\documentclass[lettersize,journal]{IEEEtran}
\usepackage{amsmath,amsfonts,amssymb,amsthm}
\usepackage{algorithmic}
\usepackage{algorithm}
\usepackage{booktabs}
\usepackage{multicol}
\usepackage{multirow}
\usepackage{booktabs}
\usepackage{enumerate}
\usepackage{xcolor}
\usepackage{colortbl}
\usepackage{tabularray}
\usepackage{microtype}
\usepackage[utf8]{inputenc}
\usepackage[T1]{fontenc}
\UseTblrLibrary{booktabs}
\usepackage{tabularx}
\usepackage{hhline}
\usepackage{paralist}
\usepackage{marvosym}
\usepackage{array}
\usepackage[caption=false,font=scriptsize,labelfont=sf,textfont=sf]{subfig}
\usepackage{textcomp}
\usepackage{stfloats}
\usepackage{soul}
\usepackage{url}
\usepackage{verbatim}
\usepackage{graphicx}
\usepackage{hhline}
\usepackage{pifont}
\usepackage{cite}
\usepackage[colorlinks,
linkcolor=black,
anchorcolor=black,
citecolor=black]{hyperref}
\usepackage{orcidlink}
\def\ie{\emph{i.e.}}
\def\eg{\emph{e.g.}}

\def\etal{{\em et al.~}}

\newtheoremstyle{nonindentdef}
{5pt}   % 上方间距
{5pt}   % 下方间距
{\normalfont}  % 正文字体
{}      % 不缩进
{\bfseries} % 定理标题字体
{.}     % 后缀
{ }     % 标题和正文之间的空格
{\thmname{#1}\thmnumber{ #2}\thmnote{ (#3)}} % 定义标题格式

\theoremstyle{nonindentdef}
\newtheorem{theorem}{\bfseries Theorem}
\newtheorem{definition}{\bfseries Definition} 
\newtheorem{assumption}{\bfseries Assumption}
\newtheorem{prop}{\bfseries Proposition} % 
\newtheorem{property}{\bfseries Property}
\newtheorem{corollary}{\bfseries Corollary}

\newcommand{\first}[1]{\textcolor{red}{\underline{\textbf{#1}}}}
\newcommand{\second}[1]{\textcolor{blue}{\textbf{#1}}}

\begin{document}

\title{Effective Receptive Field Ordering Matters for Infrared Small Target Detection}

\author{Guoyi~Zhang,~Yanjin~Du,~Zhengyao~Zhao,~Tongsu~Zhang,~Guangsheng~Xu,~Siyang~Chen,\\~Xiangpeng~Xu,~Han~Wang~and~Xiaohu~Zhang
	% <-this % stops a space
	\thanks{Corresponding authors: \emph{Xiangpeng Xu, Han Wang and Xiaohu Zhang}}
	\thanks{All authors are with the Aircraft Visual Perception (AVP) Lab, School of Aeronautics and Astronautics, Sun Yat-sen University, Shenzhen 518107, Guangdong, China. (email:\href{mailto:zhanggy57@mail2.sysu.edu.cn}{zhanggy57@mail2.sysu.edu.cn})}
}

\markboth{Journal of \LaTeX\ Class Files,~Vol.~14, No.~8, August~2021}%
{Zhang \MakeLowercase{\textit{et al.}}: Effective Receptive Field Ordering Matters for Infrared Small Target Detection}


\maketitle

\begin{abstract}
In this work, we investigate a previously unexplored architectural dimension for infrared small target detection: the organization of effective receptive fields (ERFs) during feature refinement. Unlike existing approaches that primarily improve individual feature operators, we argue that ERF organization constitutes an architectural dimension independent of receptive field design itself, and formulate deep feature transformation as a progressive residual correction process, from which a theoretical framework for ERF scheduling is established. Specifically, we reveal that ERF refinement is governed by two fundamental properties: scale-frequency correspondence, which aligns different ERF scales with distinct residual frequency characteristics, and nonlinear non-commutativity, which makes different ERF orderings produce fundamentally different refinement trajectories. Together, these properties show that ERF organization, rather than ERF scale alone, governs refinement dynamics. Guided by these principles, we propose Receptive Field Ordering Network (RFONet), which realizes hierarchical ERF scheduling through a multigrid-inspired V-cycle strategy using only standard $3\times3$ convolutions. RFONet achieves state-of-the-art performance on multiple benchmarks with only 1.16M parameters and over 157 FPS inference speed. Beyond empirical performance, our theoretical analysis provides theoretical guarantees for stable residual refinement under perturbations, frequency shifts, and partial occlusions, which are consistently reflected in superior noise robustness and cross-dataset generalization. Finally, our framework reformulates ERF organization as a task-dependent optimization objective, providing a principled foundation for future adaptive receptive field scheduling.
\end{abstract}


\section{Introduction}
\IEEEPARstart{I}{nfrared} small target detection (IRSTD) aims to identify targets with weak responses under complex background interference, which is critical for applications such as remote sensing, surveillance, and early warning. Early methods relied on handcrafted priors, including local contrast \cite{10747828}, sparsity \cite{10091756}, and low-rank properties \cite{9217948}. Benefiting from deep learning, recent approaches have substantially advanced IRSTD by developing increasingly powerful feature transformation operators, leading to remarkable improvements in feature representation and detection performance.

Despite substantial progress, existing methods primarily improve individual feature operators to enhance feature representation, while their effective receptive fields (ERFs) \cite{luo2016understanding} are implicitly determined by network depth, resolution changes, or module design. However, feature refinement is ultimately realized through the evolution of ERFs across successive nonlinear transformations, suggesting that refinement quality depends not only on individual operators but also on how their receptive fields are organized. Therefore, existing studies implicitly optimize receptive fields through operator design, whereas whether ERF organization itself should be explicitly optimized as an independent architectural dimension remains largely unexplored.

In this work, we investigate ERF organization as a previously overlooked architectural dimension for infrared small target detection. Rather than treating ERF scales as passive consequences of network depth, resolution variation, or module configuration, we formulate deep feature transformation as a progressive residual refinement process and establish a theoretical framework for ERF scheduling. Our analysis reveals two fundamental properties. First, ERF refinement exhibits scale-frequency correspondence, indicating that different ERF scales preferentially refine residual components with different frequency characteristics. Second, nonlinear ERF transformations are generally non-commutative, implying that different ERF orderings produce fundamentally different refinement trajectories. Together, these properties indicate that refinement performance is governed not only by ERF scale but also by how ERFs are organized throughout the refinement process.

Guided by these theoretical principles, we propose Receptive Field Ordering Network (RFONet), a minimalist framework that explicitly organizes effective receptive fields during progressive feature refinement. Using only standard $3\times3$ convolutions, RFONet realizes hierarchical ERF scheduling through a V-cycle strategy inspired by multigrid optimization \cite{11483348}, allowing the effect of ERF organization to be isolated from architectural complexity. This design translates the proposed principles into stable scale-consistent refinement with improved robustness to perturbations, frequency shifts, and partial occlusions. Extensive experiments consistently support these theoretical predictions, demonstrating state-of-the-art accuracy together with superior robustness, cross-dataset generalization, and high computational efficiency. More importantly, our analysis suggests that the optimal ERF organization is determined by task-specific residual characteristics, motivating receptive field scheduling as a principled optimization objective rather than a heuristic architectural design choice.

The contributions of this work are summarized as follows:
\begin{itemize}
	\item We identify ERF organization as a previously unexplored architectural principle for infrared small target detection, shifting the focus from individual operator design to receptive field scheduling during progressive refinement.
	\item We establish a theoretical framework revealing the scale-frequency correspondence and nonlinear non-commutativity of ERF refinement, and derive its stability and robustness properties.
	\item We propose RFONet, a minimalist realization of the proposed principles through hierarchical V-cycle ERF scheduling, achieving accurate, robust, and highly efficient infrared small target detection.
\end{itemize}
The rest of this paper is organized as follows. Section~\ref{Section:Related_Work} reviews related works; Section~\ref{Section:Problem_Formulation} presents the problem formulation and theoretical analysis; Section~\ref{Section:ProposedMethod} introduces RFONet; Section~\ref{Section:Experiment} reports experimental results; and Section~\ref{Section:Conclusion} concludes the paper.
\section{Related Work}\label{Section:Related_Work}
\subsection{Infrared Small Target Detection}
With the rapid development of deep learning, data-driven methods have become the dominant paradigm for infrared small target detection \cite{11479915}. Most existing studies improve detection performance by enhancing the capability of individual feature transformation operators. Representative efforts include multi-scale feature interaction (DNANet \cite{DNANet}, UIUNet \cite{UIUNet}), shape-biased representation learning (ISNet \cite{ISNet}, CSRNet \cite{CSRNet}), global dependency modeling (SRSNet \cite{11174084}, ISTR \cite{liu2023infrared}, MCFNet \cite{11568953}), and prior-guided enhancement strategies based on frequency modeling, multi-task learning, foundation models, and unfolding mechanisms (WMRNet \cite{11278553}, MTMLNet \cite{11080263}, IRSAM \cite{IRSAM}, RPCANet$^{++}$ \cite{11537388}). Despite these advances, existing methods focus on improving individual feature operators but neglect the sequential organization of receptive fields during refinement, leaving the role of ERF scheduling in detection performance largely unclear.
\subsection{Effective Receptive Field Modeling}
Effective receptive field (ERF) has been widely recognized as a key factor in contextual information aggregation \cite{11440142,10962317,9186840}. Increasing ERF enables models to capture broader spatial dependencies, motivating extensive efforts to enlarge ERFs through multi-scale architectures \cite{10872821}, non-local operations \cite{wu2025lrformer}, self-attention mechanisms \cite{liu2023infrared}, large-kernel convolutions \cite{11130640}, and recent state-space models such as Mamba \cite{mamba2}. Beyond spatial coverage, existing studies have also investigated the semantic roles of different ERF scales \cite{11304568}, revealing that smaller ERFs are more suitable for preserving local details, whereas larger ERFs facilitate global structural modeling \cite{hou2024conv2former}.
Deformable convolutions \cite{xiong2024efficient} provide adaptive receptive fields through spatial sampling adjustment (where), yet leave the scale-wise organization of ERFs (when/how) unexplored. However, these studies mainly investigate individual ERF scales or their magnitude \cite{11498646}, while how multiple ERFs should be organized and scheduled during progressive feature refinement remains largely unexplored.
\section{Problem Formulation and Motivation}\label{Section:Problem_Formulation}
\subsection{Problem Formulation}
Given an infrared image, infrared small target detection aims to estimate the target map. The observed image is formulated as \cite{gao2013infrared}
\begin{equation}
	f_D(x,y) = f_T(x,y) + f_B(x,y) + f_N(x,y),
\end{equation}
where $f_D$, $f_T$, $f_B$, $f_N$ and $(x,y)$ are the original infrared image, the target image, the background image, the random noise image and the pixel location, respectively.  A deep network can be viewed as a finite-step refinement process \cite{8727950}:
\begin{equation}
	\mathbf{F}^{(k+1)}=\Phi_k(\mathbf{F}^{(k)}),\quad k=0,\dots,K-1,
\end{equation}
where $\mathbf{F}^{(k)}$ denotes the intermediate representation. We formulate it as $\mathbf{F}^{(k)}=\mathbf{T}^{(k)}+\mathbf{E}^{(k)}$, where $\mathbf{T}^{(k)}$ and $\mathbf{E}^{(k)}$ represent the target-related component and residual ambiguity, respectively. Thus, feature refinement aims to reduce $\mathbf{E}^{(k)}$ while preserving discriminative target information.

\subsection{Problem Analysis}
Based on the above formulation, the residuals involved in feature refinement are further analyzed from the perspective of their scale-frequency characteristics.
\begin{definition}[Scale-frequency Error Decomposition]
	Let an error signal $e\in\mathcal{E}$ be decomposed into multiple scale-frequency components $e=\sum_{s=1}^{S}e_s$,	where $e_s$ denotes the error component dominated by the frequency band
	$\Omega_s$. We assume a scale-frequency correspondence, where finer spatial scales are associated with higher-frequency components and coarser spatial scales correspond to lower-frequency components, \ie, $s_i<s_j \Rightarrow \Omega_{s_i}\succ\Omega_{s_j}$.
\end{definition}

\begin{assumption}[Scale-selective Stable Refinement]
	\label{Ass:SSR}
	Let $\mathcal T_s$ denote the refinement operator associated with scale $s$. For the residual component $e_s$ dominated by frequency band $\Omega_s$,
	$\|\mathcal T_s(e_s)\|
	\le
	\alpha_s\|e_s\|,
	\;
	0\le\alpha_s<1$,
	while for any mismatched residual component $e_j$ $(j\neq s)$,
	$\|\mathcal T_s(e_j)\|
	\le
	\beta_{s,j}\|e_j\|,
	\;
	\alpha_s<\beta_{s,j}<1$.
	Furthermore, for any two refinement operators $\Phi_k$ and $\Phi_s$, the correction discrepancy is bounded by
	$\|
	\Phi_k(e_s)-\Phi_s(e_s)
	\|
	\le
	\zeta_{k,s}\|e_s\|,
	\;
	\zeta_{k,s}<\infty$.
\end{assumption}
\noindent\textbf{Remark.}
This is not a restrictive condition but a formal abstraction of two universal facts: spatial scale inherently determines frequency preference, and any two bounded neural operators have finite Lipschitz discrepancy on a compact input domain.
\begin{theorem}[Scale-matched Residual Refinement]
	\label{Thm:SMS}
	Let a residual component $e_s$ correspond to frequency scale $s$. Given a refinement operator $\Phi_s$ specialized for correcting the residual component at scale $s$ and an arbitrary mismatched operator	$\Phi_k$ ($k\neq s$) applied to the same residual component, if the	cross-scale interference satisfies Assumption~\ref{Ass:SSR}, then
	$\left\|
	e_s-\Phi_s(e_s)
	\right\|
	\leq
	\left\|
	e_s-\Phi_k(e_s)
	\right\|
	+
	\zeta_{k,s}\|e_s\|$.
	Therefore, scale-matched refinement reduces the uncertainty introduced by applying an inappropriate correction operator to a residual component.
\end{theorem}
\begin{proof}
	For the matched refinement operator, the corrected residual is defined as
	\begin{equation}
		r_s^m=e_s-\Phi_s(e_s).
	\end{equation}
	For a mismatched refinement operator $\Phi_k$, the corresponding residual is
	\begin{equation}
		r_s^u=e_s-\Phi_k(e_s).
	\end{equation}
	By adding and subtracting $\Phi_k(e_s)$, we have
	\begin{equation}
		\begin{aligned}
			r_s^m
			&=
			e_s-\Phi_k(e_s)
			+
			\Phi_k(e_s)-\Phi_s(e_s).
		\end{aligned}
	\end{equation}
	Applying the triangle inequality gives
	\begin{equation}
		\|r_s^m\|
		\leq
		\|r_s^u\|
		+
		\|\Phi_k(e_s)-\Phi_s(e_s)\|.
	\end{equation}
	According to Assumption~\ref{Ass:SSR}, the discrepancy between two refinement operators is bounded as
	\begin{equation}
		\|\Phi_k(e_s)-\Phi_s(e_s)\|
		\leq
		\zeta_{k,s}\|e_s\|.
	\end{equation}
	Combining the above inequalities yields
	\begin{equation}
		\|r_s^m\|
		\leq
		\|r_s^u\|
		+
		\zeta_{k,s}\|e_s\|,
	\end{equation}
	which completes the proof.
\end{proof}
\noindent\textbf{Remark.}
Theorem~\ref{Thm:SMS} does not state that mismatched operators are ineffective, but only that scale mismatch introduces an additional bounded interference term, motivating scale-aware organization of refinement operators.
\subsection{Motivation}
The above analysis suggests that residual correction is essentially a scale-matching process, where the spatial characteristics of correction operators determine their ability to handle different error components. In deep networks, such spatial characteristics are governed by the effective receptive field \cite{luo2016understanding} of feature transformations.
\begin{definition}[Effective Receptive Field (ERF \cite{luo2016understanding}, Luo \etal NeurIPS'16)]\label{def:erf}
	The effective receptive field (ERF) characterizes the actual contribution distribution of input regions to an output response, rather than the maximum coverage defined by the theoretical receptive field. For an $L$-layer network, the contribution of an input position $(i-p_1,j-p_2)$ to an output neuron at $(i,j)$ is measured by
	$g(p_1,p_2)
	=
	\left\|
	\frac{\partial \mathbf{Y}^{(L)}(i,j)}
	{\partial \mathbf{X}^{(0)}(i-p_1,j-p_2)}
	\right\|_1$.
	The ERF is defined as the normalized contribution distribution:
	$\mathrm{ERF}(p_1,p_2)
	=
	\frac{g(p_1,p_2)}
	{\sum_{q_1,q_2}g(q_1,q_2)}$.
	The spatial distribution of $\mathrm{ERF}(p_1,p_2)$ reflects the effective scale of information aggregation.
\end{definition}
\begin{property}[Frequency Preference Induced by ERF Scale]
	\label{prop:ERF_frequency}
	Let $\rho_{\sigma}(\mathbf{x})$ denote an ERF distribution with spatial	scale $\sigma$, where $\rho_{\sigma}(\mathbf{x})
	=
	\frac{1}{\sigma^d}
	\rho
	(
	\frac{\mathbf{x}}{\sigma}
	)$,
	where $d$ represents the spatial dimension.	According to the Fourier scaling property, the corresponding frequency response satisfies 
	$\hat{\rho}_{\sigma}(\boldsymbol{\omega})
	=
	\hat{\rho}
	(
	\sigma\boldsymbol{\omega}
	)$.
	Therefore, increasing ERF scale expands the spatial aggregation range and shifts the frequency response toward lower frequencies. Specifically, for	$\sigma_2>\sigma_1$, the response of a larger ERF is evaluated at a higher frequency coordinate $\sigma_2\|\boldsymbol{\omega}\|
	>
	\sigma_1\|\boldsymbol{\omega}\|$,
	which results in a relatively stronger attenuation tendency for rapidly	varying components. Hence, larger ERFs exhibit stronger low-frequency preference, whereas smaller ERFs preserve more localized high-frequency details.
\end{property}
\begin{proof}
	The spatial scaling relationship of ERF indicates that enlarging the receptive field is equivalent to stretching its contribution distribution.	According to the Fourier scaling theorem \cite{huang2023adaptive},
	\begin{equation}
		\mathcal{F}
		\{
		\rho_{\sigma}(\mathbf{x})
		\}
		=
		\hat{\rho}
		(
		\sigma\boldsymbol{\omega}
		).
	\end{equation}
	When $\sigma$ increases, the frequency response is compressed toward the low-frequency region.
\end{proof}
\begin{figure}[!t]
	\centering
	\includegraphics[width=1\linewidth]{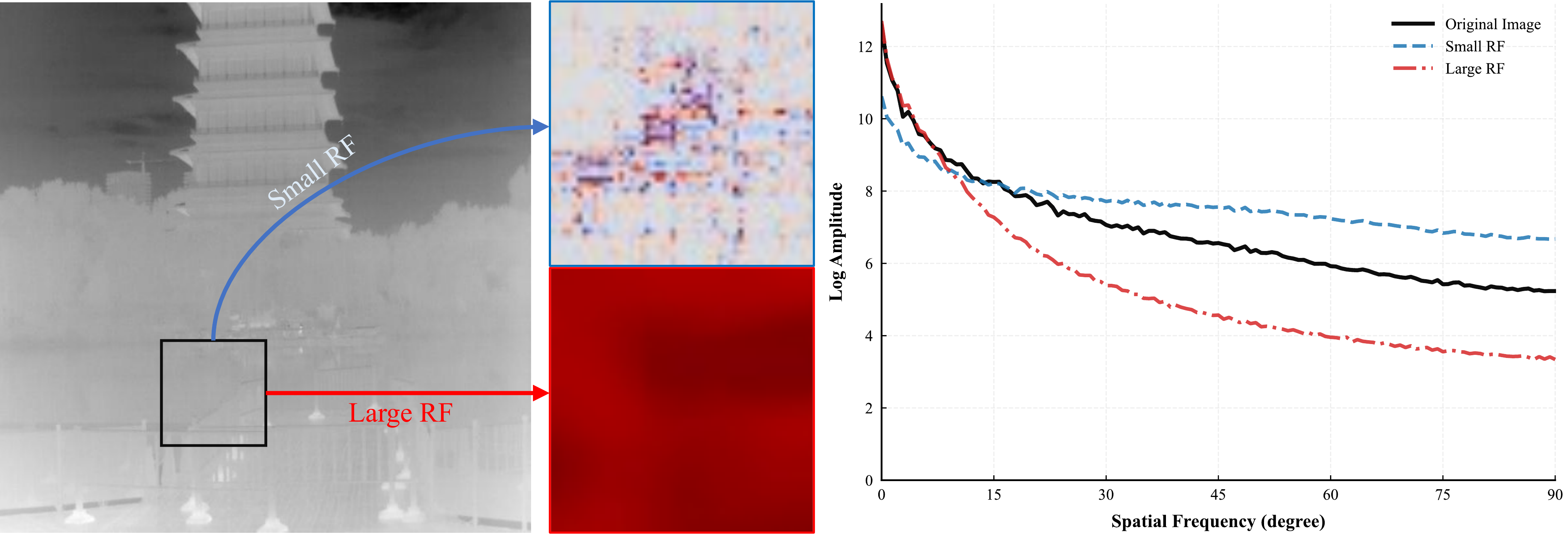}
	\caption{Verification of Property \ref{prop:ERF_frequency}. Left: original image with local pseudo-color overlays of small and large receptive field responses. Right: radial frequency spectra. Small RF preserves high frequencies; large RF attenuates them.}
	\label{fig:fre}
\end{figure}
\noindent\textbf{Remark.} Property~\ref{prop:ERF_frequency} indicates that ERF is not merely a measure of spatial coverage, but also characterizes the frequency preference of feature transformations. Specifically, as shown in Fig. \ref{fig:fre}, large ERFs facilitate broader contextual aggregation and are more suitable for correcting low-frequency residual components, whereas small ERFs preserve sensitivity to localized high-frequency variations. Therefore, different ERF scales can be interpreted as scale-specific refinement operators, providing a theoretical basis for ERF scheduling.
\begin{property}[Non-commutativity of ERF-guided Refinement]
	\label{prop:ERF_noncomm}
	Let $\Phi_{r_i}$ and $\Phi_{r_j}$ denote two feature transformation	operators associated with different ERF scales $r_i$ and $r_j$. For nonlinear deep networks, these ERF-guided transformations generally do not satisfy the commutative property $\Phi_{r_i}\circ\Phi_{r_j}
	\neq
	\Phi_{r_j}\circ\Phi_{r_i}$.
	Therefore, the ordering of ERF transformations affects the feature refinement trajectory and the final representation.
\end{property}
\begin{proof}
	An ERF-guided transformation with scale $r$ can be formulated as
	\begin{equation}
		\Phi_r(\mathbf{x})
		=
		\sigma
		(
		\mathcal K_r(\mathbf{x})
		),
	\end{equation}
	where $\mathcal K_r(\cdot)$ denotes the scale-dependent aggregation	operator induced by ERF $r$, and $\sigma(\cdot)$ represents nonlinear feature transformation. For two different ERF scales $r_i$ and $r_j$, the corresponding aggregation operators generally satisfy
	\begin{equation}
		\mathcal K_{r_i}
		\neq
		\mathcal K_{r_j},
	\end{equation}
	because they integrate information from different spatial neighborhoods. Considering two refinement orders, we obtain
	\begin{equation}
		\Phi_{r_i}
		(
		\Phi_{r_j}(\mathbf{x})
		)
		=
		\sigma
		(
		\mathcal K_{r_i}
		(
		\sigma
		(
		\mathcal K_{r_j}(\mathbf{x})
		)
		)
		),
	\end{equation}
	and
	\begin{equation}
		\Phi_{r_j}
		(
		\Phi_{r_i}(\mathbf{x})
		)
		=
		\sigma
		(
		\mathcal K_{r_j}
		(
		\sigma
		(
		\mathcal K_{r_i}(\mathbf{x})
		)
		)
		).
	\end{equation}
	Since each transformation modifies the feature distribution for subsequent operators and ERFs operate at different spatial scales, different ordering strategies generally yield distinct intermediate representations.
	Therefore,
	\begin{equation}
		\Phi_{r_i}\circ\Phi_{r_j}(\mathbf{x})
		\neq
		\Phi_{r_j}\circ\Phi_{r_i}(\mathbf{x}),
	\end{equation}
	in general.
\end{proof}
\begin{figure}[!t]
	\centering
	\includegraphics[width=1\linewidth]{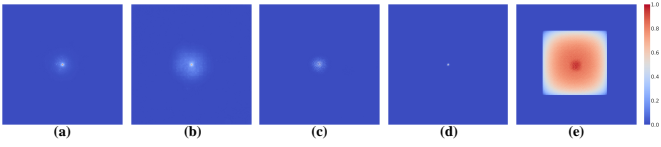}
	\caption{Qualitative verification of the non-commutativity property. Five MSHNet \cite{MSHNet} variants are constructed by inserting HLKConv modules \cite{LCRNet} at different locations. They contain identical operators, and parameters, differing only in the ordering of receptive field expansion. Nevertheless, their input-to-output effective receptive fields are markedly different, demonstrating that ERF ordering alone changes the refinement trajectory, consistent with Property \ref{prop:ERF_noncomm}.}
	\label{fig:Comm}
\end{figure}
\noindent\textbf{Remark.}
Property~\ref{prop:ERF_noncomm} establishes the theoretical necessity of
ERF scheduling, while Fig.~\ref{fig:Comm} provides qualitative evidence that
different ERF orderings produce distinct effective receptive fields.

\noindent\textbf{Motivation.}
Combining Definition~\ref{def:erf}, Property~\ref{prop:ERF_frequency}, Theorem~\ref{Thm:SMS}, and Property~\ref{prop:ERF_noncomm}, we reveal that ERF scheduling is fundamentally a scale-frequency matching problem rather than a simple architectural design choice. Specifically, different ERF scales exhibit distinct frequency preferences, enabling scale-adaptive refinement of heterogeneous residual components. Furthermore, the non-commutative nature of nonlinear ERF-guided transformations indicates that different ERF orders produce distinct refinement trajectories and final representations. Therefore, effective ERF-based refinement depends not only on the available scales but also on their organization throughout the refinement process. This leads to a fundamental question: how should ERFs with different spatial scales be scheduled to achieve effective residual suppression?

\section{Methodology}\label{Section:ProposedMethod}
\begin{figure*}[!t]
	\centering
	\includegraphics[width=1\linewidth]{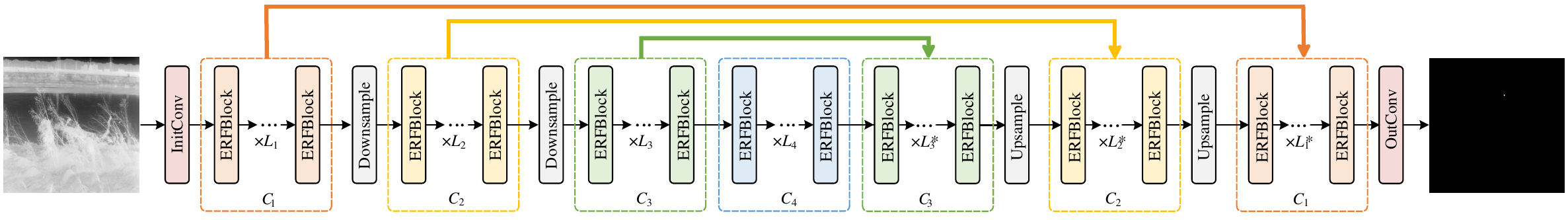}
	\caption{The overall architecture of RFONet follows the conventional U-Net framework with stacked ERFBlocks. InitConv and OutConv are implemented by two $3\times3$ convolutions, projecting the input from 1 to $C_1$ channels and the features back to a single-channel prediction, respectively. Here, $L_i^*=L_i+1$. To isolate the effect of ERF organization, all ERFBlocks use the same local operator, and different ERFs are generated solely through resolution transformations.}
	\label{fig:OverallofRFONet}
\end{figure*}
\subsection{Overview Architecture}
To isolate the effect of ERF organization from architectural complexity, we adopt a minimalist design principle. As shown in Fig.~\ref{fig:OverallofRFONet}, RFONet follows a standard U-Net architecture without introducing additional feature extraction branches or task-specific modules commonly used in infrared small target detection. Each stage is constructed with the same ERFBlock, which only contains standard $3\times3$ convolutions, pooling, nearest-neighbor interpolation, normalization, and nonlinear activation. Different ERF scales are naturally induced by varying spatial resolutions through hierarchical downsampling. Therefore, RFONet serves as a controlled framework to investigate how ERFs should be organized across refinement stages, rather than relying on architectural complexity for performance improvement.

\subsection{Simple Yet Powerful ERFBlock}
Infrared small target detection requires suppressing heterogeneous residuals, including high-frequency disturbances from noise and edges and low-frequency variations from large-scale background structures \cite{gao2013infrared}. According to Theorem \ref{Thm:SMS} and Property \ref{prop:ERF_frequency}, different ERFs exhibit distinct frequency preferences and therefore specialize in refining different residual components. Consequently, effective feature refinement depends not only on the available ERF scales but also on how they are organized throughout the refinement process. This naturally motivates a multiscale scheduling framework in which different ERF scales cooperate to progressively suppress heterogeneous residuals. Following this principle, we propose the Effective Receptive Field Block (ERFBlock), which explicitly organizes ERFs during progressive feature refinement.
\begin{definition}[Multigrid Principle]\label{Def:MG} 
	Let $\{\mathcal{C}_{s}\}_{s=1}^{S}$ denote a set of scale-specific correction operators, where $s$ indexes the resolution scale. Given an error representation $e^{(0)}$, the refinement process under a selected scale sequence $\{s_k\}_{k=1}^{K}$ is formulated as $e^{(k+1)} = (\mathcal{I}-\mathcal{C}_{s_k})e^{(k)}, \; k=0,\cdots,K-1 $, where $\mathcal{C}_{s_k}$ selectively attenuates the error components corresponding to scale $s_k$. The final residual after $K$ correction steps can be expressed as $e^{(K)} = \prod_{k=1}^{K} (\mathcal{I}-\mathcal{C}_{s_k})e^{(0)} $. Therefore, the multigrid principle aims to determine an effective scale scheduling strategy: $\{s_k^{*}\}_{k=1}^{K} = \arg\min_{\{s_k\}} \left\| \prod_{k=1}^{K} (\mathcal{I}-\mathcal{C}_{s_k})e^{(0)} \right\|$, where different scales cooperate to reduce heterogeneous error components. 
\end{definition}
Although the above formulation is derived in the residual space, it naturally extends to deep feature refinement, where ERFs act as scale-specific refinement operators. The multigrid principle admits a variety of scheduling strategies, ranging from the classical V-cycle and W-cycle to more sophisticated adaptive schedules. While these alternatives offer different levels of scheduling flexibility, they also introduce additional optimization factors that may obscure the contribution of ERF organization itself. \textbf{\textit{Since the objective of this work is to investigate the role of ERF ordering rather than the scheduling policy, we deliberately adopt the classical V-cycle as the simplest yet representative instantiation of the multigrid principle.}} Compared with adaptive schedules, this fixed design introduces no additional optimization variables \cite{9560049}, providing a controlled setting to isolate the contribution of ERF ordering.
\begin{definition}[V-cycle Refinement Principle]
	\label{Def:Vcycle}
	Let $\{\mathcal C_s\}_{s=1}^{S}$ denote scale-specific correction operators with ERF size $r_s$. A V-cycle schedule is defined as $\mathcal S_V=
	\{s_1,\cdots,s_m,\cdots,s_1\}$,
	where $r_{s_1}<r_{s_2}<\cdots<r_{s_m}$,	and the corresponding refinement process is $e^{(k+1)}
	=
	(\mathcal I-\mathcal C_{s_k})e^{(k)}$.
	The complete V-cycle operator is
	$\mathcal V
	=
	\prod_{k=1}^{K}
	(\mathcal I-\mathcal C_{s_k})$,
	yielding
	$e^{(K)}
	=
	\mathcal V e^{(0)}$.
\end{definition}

\noindent\textbf{Implementation Details.}
Following Definition~\ref{Def:Vcycle}, we instantiate the proposed ERF organization through a scale-scheduled residual refinement mechanism. Different from conventional multi-scale designs that focus on simultaneous feature aggregation, each ERFBlock explicitly organizes receptive field transitions into a sequential refinement process across different spatial
scales. For a maximum scale level \(s\), the refinement trajectory is defined as
\begin{equation}
	\mathcal{S}_{s}
	=
	\{1,2,\cdots,s-1,s,s-1,\cdots,2,1\},
\end{equation}
which progressively expands the contextual modeling range and subsequently returns to fine-scale refinement. Based on this trajectory, the overall ERF refinement operator is formulated as
\begin{equation}
	\mathcal{V}_{s}
	=
	\mathcal{R}_{s_K}
	\circ
	\cdots
	\circ
	\mathcal{R}_{s_1},
	\quad s_k\in\mathcal{S}_{s},
	\label{Eq:s}
\end{equation}
where $\mathcal{R}_{s_k}$ denotes the scale-conditioned ERF refinement operator. For instance,
$\mathcal{S}_{1}=\{1\}$,
$\mathcal{S}_{2}=\{1,2,1\}$,
and
$\mathcal{S}_{3}=\{1,2,3,2,1\}$.
Such hierarchical scheduling enables residual correction along a progressive multi-scale refinement trajectory. Specifically, each ERF refinement operator is formulated as a residual
transformation in a scale-dependent spatial domain:
\begin{equation}
	\mathcal{R}_{s}(\mathbf X)
	=
	\mathbf X
	+
	\mathcal{T}_{s}(\mathbf X),
\end{equation}
where $\mathcal{T}_{s}(\cdot)$ denotes the scale-conditioned residual correction operator. This formulation directly realizes the scale-matched correction principle in Theorem~\ref{Thm:SMS}: each $\mathcal{R}_s$ is specialized for residual components at scale $s$, with cross-scale interference bounded by domain projection. To explicitly control the ERF size, $\mathcal{T}_{s}$ is implemented as
\begin{equation}
	\mathcal{T}_{s}
	=
	\Pi^{-1}_{s}
	\circ
	\mathcal{F}
	\circ
	\Pi_{s},
\end{equation}
where $\Pi_s(\cdot)$ projects the feature representation into a scale-specific spatial domain, $\Pi_s^{-1}(\cdot)$ restores the original resolution, and $\mathcal{F}(\cdot)$ performs local feature refinement. Through different spatial projections, the same refinement function generates distinct ERF scales without introducing additional convolutional parameters, consistent with the multigrid principle of sharing refinement operators across scales \cite{10061442}.
In practice, the shared refinement function is implemented as
\begin{equation}
	\mathcal{F}(\mathbf X)
	=
	\lambda
	\mathrm{Conv}_{3\times3}
	(
	\delta(\mathrm{Norm}(\mathbf X))
	),
\end{equation}
where $\lambda$ is a learnable modulation coefficient. The scale projection $\Pi_s(\cdot)$ and inverse projection $\Pi_s^{-1}(\cdot)$ are realized by max pooling and nearest-neighbor
interpolation with a scale factor of $2^s$, respectively. For an input feature map $\mathbf X\in\mathbb{R}^{C\times H\times W}$, the complete ERFBlock is formulated as a hierarchical residual refinement:
\begin{equation}
	\hat{\mathbf X}
	=
	\mathcal{V}_{s}(\mathbf X)
	+
	\mathcal{F}(\mathbf X),
\end{equation}
\begin{equation}
	\mathbf Y
	=
	\mathbf X
	+
	\mathcal{F}(\hat{\mathbf X}).
\end{equation}
By decoupling ERF scheduling from feature transformation, the proposed ERFBlock provides a controlled framework for investigating how receptive field organization influences progressive feature refinement.

\noindent\textbf{Remark.} Let $s$ denote the effective stride between a feature map and the input image. Since each feature element corresponds to an $s\times s$ region in the input domain, the spatial scale of the ERF is proportionally enlarged by the stride factor \cite{luo2016understanding}:
\begin{equation}
	r_{\mathrm{ERF}}^{(s)}
	=
	s\cdot r_{\mathrm{ERF}}^{(1)} .
\end{equation}
Therefore,
\begin{equation}
	s_2>s_1
	\Rightarrow
	r_{\mathrm{ERF}}^{(s_2)}
	>
	r_{\mathrm{ERF}}^{(s_1)} ,
\end{equation}
indicating that down-sampling provides a simple mechanism to enlarge ERF without modifying the local operator.
\begin{prop}[ERFBlock as a Multigrid V-Cycle]
	Let $\{\mathcal{V}_\ell\}_{\ell=1}^{s}$ be a hierarchy of feature spaces.
	Define operators:
	$R_\ell:\mathcal{V}_\ell\to\mathcal{V}_{\ell+1},\quad
	P_\ell:\mathcal{V}_{\ell+1}\to\mathcal{V}_\ell,\quad
	S_\ell:\mathcal{V}_\ell\to\mathcal{V}_\ell,$ 
	with $R_s=P_s=I$. Define the single-level correction:
	$\mathcal{R}_\ell = I + P_\ell \circ S_\ell \circ R_\ell$.
	The ERFBlock V-cycle along $\mathcal{S}_s=\{1,2,\dots,s,\dots,1\}$ is:
	$\mathcal{V}_s =
	\left(\prod_{\ell=1}^{s} \mathcal{R}_\ell\right)
	\circ
	\left(\prod_{\ell=s-1}^{1} \mathcal{R}_\ell\right)$.
	Expanding:
	$\mathcal{V}_s =
	\mathcal{R}_1\mathcal{R}_2\cdots\mathcal{R}_{s-1}
	\;\mathcal{R}_s\;
	\mathcal{R}_{s-1}\cdots\mathcal{R}_2\mathcal{R}_1$.
	The classical multigrid V-cycle operator $V_\ell:\mathcal{V}_\ell\to\mathcal{V}_\ell$ is:
	$	V_\ell = S_\ell, \ell=s,
	S_\ell \circ (I + P_\ell \circ V_{\ell+1} \circ R_\ell) \circ S_\ell, \ell<s.$
	\textbf{Claim:} Both follow the identical recursion:
	$	\Phi_\ell = \mathcal{R}_\ell, \ell=s,
	\mathcal{R}_\ell \circ \Phi_{\ell+1} \circ \mathcal{R}_\ell, \ell<s$.
	Since $\mathcal{V}_s = \Phi_1$ by expansion, and $V_1 = \Phi_1$ by definition,
	we have $\mathcal{V}_s = V_1$.
\end{prop}

\begin{theorem}[Scale-matched Residual Contraction]
	\label{Thm:ERFBlock_refinement}
	Consider the progressive residual refinement process of an ERF-guided network. The evolution of residual component $E_s$ at refinement stage $k$ is	formulated as
	$E_s^{(k+1)}
	=
	E_s^{(k)}
	-
	\eta_{s_i,s}\Phi_{s_i}(E_s^{(k)})
	+
	\xi_s(E^{(k)})$,
	where $\Phi_{s_i}$ denotes the ERF correction operator with scale $s_i$, $\eta_{s_i,s}$ represents the correction strength, and	$\xi_s(E^{(k)})$ denotes the aggregated nonlinear residual perturbation on	component $s$, including self-interaction and cross-scale effects. Under the bounded cross-scale interference condition established in Theorem~\ref{Thm:SMS}, there exists a local constant $L_s>0$ satisfying
	$\|\xi_s(E^{(k)})\|
	\leq
	L_s\|E_s^{(k)}\|$.
	Furthermore, assume that the scale-matched ERF correction step forms a local contraction, \ie, there exists a contraction coefficient
	$\rho_{s_i,s}\in[0,1)$ such that
	$\left\|
	E_s^{(k)}
	-
	\eta_{s_i,s}\Phi_{s_i}(E_s^{(k)})
	\right\|
	\leq
	\rho_{s_i,s}
	\|E_s^{(k)}\|$.
	If
	$\rho_{s_i,s}+L_s<1$,
	then the residual refinement process is contractive and satisfies
	$\|E_s^{(k+1)}\|
	\leq
	(\rho_{s_i,s}+L_s)
	\|E_s^{(k)}\|$.
	Therefore, the scale-matched ERF organization guarantees stable residual attenuation under nonlinear feature transformations.
\end{theorem}

\begin{proof}
	According to the refinement formulation,
	\begin{equation}
		\begin{aligned}
			\|E_s^{(k+1)}\|
			&=
			\|
			E_s^{(k)}
			-
			\eta_{s_i,s}\Phi_{s_i}(E_s^{(k)})
			+
			\xi_s(E^{(k)})
			\|.
		\end{aligned}
	\end{equation}
	Applying the triangle inequality gives
	\begin{equation}
		\begin{aligned}
			\|E_s^{(k+1)}\|
			&\leq
			\left\|
			E_s^{(k)}
			-
			\eta_{s_i,s}\Phi_{s_i}(E_s^{(k)})
			\right\|
			+
			\|\xi_s(E^{(k)})\|.
		\end{aligned}
	\end{equation}
	According to the local contraction property of the scale-matched ERF correction,
	\begin{equation}
		\left\|
		E_s^{(k)}
		-
		\eta_{s_i,s}\Phi_{s_i}(E_s^{(k)})
		\right\|
		\leq
		\rho_{s_i,s}\|E_s^{(k)}\|,
	\end{equation}
	and the bounded nonlinear perturbation satisfies
	\begin{equation}
		\|\xi_s(E^{(k)})\|
		\leq
		L_s\|E_s^{(k)}\|.
	\end{equation}
	Combining the above inequalities yields
	\begin{equation}
		\begin{aligned}
			\|E_s^{(k+1)}\|
			&\leq
			(\rho_{s_i,s}+L_s)
			\|E_s^{(k)}\|.
		\end{aligned}
	\end{equation}
	When $\rho_{s_i,s}+L_s<1$, the refinement process becomes a contraction	mapping, indicating that the residual component decreases progressively.
\end{proof}
\noindent\textbf{Remark.}
Theorem~\ref{Thm:ERFBlock_refinement} establishes the stability criterion of ERFBlock refinement. Stable residual attenuation is guaranteed whenever the scale-matched contraction dominates nonlinear perturbation ($\rho_{s_i,s}+L_s<1$), independent of the specific operator realization. In practice, the contraction condition is implicitly enforced by the normalization layers and residual connections \cite{NEURIPS2019_7716d0fc}, and its validity is further supported by the stable convergence observed in our experiments.

\noindent\textbf{Complexity Analysis.}
All ERF scales share the same refinement function $\mathcal{F}$, yielding
\begin{equation}
	\mathcal{P}_{ERF}
	=
	\mathcal{P}(\mathcal{F}),
	\quad
	\frac{\partial \mathcal{P}_{ERF}}{\partial s}=0 .
\end{equation}
The computational cost of the V-cycle refinement is
\begin{equation}
	\mathcal{O}(\mathcal{V}_{s})
	=
	\sum_{k=1}^{K}
	\mathcal{O}(\mathcal{R}_{s_k}),
	\quad
	\mathcal{O}(\mathcal{R}_{s})
	=
	\mathcal{O}
	\left(
	\frac{HW}{4^{s}}C^{2}
	\right).
\end{equation}
Thus, the proposed ERF organization enlarges receptive fields with constant parameter complexity and limited computational overhead.

\subsection{Relation to Adaptive Receptive Field Learning}
Adaptive receptive field learning and RFONet optimize different architectural variables. Adaptive methods optimize the sampling geometry of individual operators \cite{xiong2024efficient},
\begin{equation}
	\mathbf y(p)
	=
	\sum_{i\in\Omega}
	w_i\mathbf x(p+\Delta p_i),
	\quad
	\Theta_{\rm adaptive}
	=
	\{\Delta p_i\},
\end{equation}
whereas RFONet keeps local operators unchanged and optimizes only the ERF scheduling trajectory,
\begin{equation}
	\mathcal V
	=
	\mathcal R_{s_K}
	\circ
	\cdots
	\circ
	\mathcal R_{s_1},
	\quad
	\Theta_{\rm RFONet}
	=
	\{s_1,\cdots,s_K\}.
\end{equation}
Hence, the former optimizes \emph{operator behavior}, whereas the latter optimizes \emph{operator organization}. For an equivalent receptive field radius $R$, adaptive methods enlarge the local sampling support,
\begin{equation}
	|\Omega|
	=
	\Theta(R^2),
	\quad
	\mathcal C_{\rm adaptive}
	=
	\Theta(HWC^2R^2),
\end{equation}
whereas RFONet enlarges receptive fields through hierarchical domain projection,
\begin{equation}
	R
	=
	2^sk,
	\quad
	\mathcal C_{\rm RFONet}
	=
	\Theta
	\left(
	\frac{HW}{4^s}C^2k^2
	\right).
\end{equation}
Therefore,
\begin{equation}
	\frac{\mathcal C_{\rm adaptive}}
	{\mathcal C_{\rm RFONet}}
	=
	\Theta
	\left(
	\frac{R^4}{k^4}
	\right),
\end{equation}
showing that equivalent ERFs can be obtained through fundamentally different computational mechanisms. More importantly, enlarging adaptive receptive fields simultaneously enlarges the optimization space,
\begin{equation}
	|\Omega|
	=
	\Theta(R^2)
	\Longrightarrow
	\dim(\Theta_{\rm adaptive})
	=
	\Theta(R^2),
\end{equation}
whereas
\begin{equation}
	\dim(\Theta_{\rm RFONet})
	=
	K,
\end{equation}
which is independent of the target ERF radius. Consequently, optimization difficulty increases together with ERF size for adaptive operators unless dedicated optimization techniques (\eg, re-parameterization \cite{11069297}) are introduced, whereas RFONet enlarges receptive fields by organizing refinement trajectories while preserving fixed local operators. Therefore, adaptive receptive field learning and ERF organization should be viewed as complementary optimization paradigms rather than alternative implementations for enlarging receptive fields.

\subsection{Loss Function}
To handle the severe imbalance between target and background pixels, the Soft-IoU loss is adopted to optimize the predicted confidence map:
\begin{equation}
	\mathcal{L}_{\mathrm{IoU}}
	=
	1-
	\frac{\sum_{i,j}\mathbf{X}_{i,j}\mathbf{Y}_{i,j}}
	{\sum_{i,j}\mathbf{X}_{i,j}+\mathbf{Y}_{i,j}
		-\mathbf{X}_{i,j}\mathbf{Y}_{i,j}},
\end{equation}
where $\mathbf{X}$ and $\mathbf{Y}$ denote the predicted confidence map and ground truth, respectively. In addition, an $\ell_1$ sparsity regularization is introduced to constrain false activations:
\begin{equation}
	\mathcal{L}
	=
	\mathcal{L}_{\mathrm{IoU}}
	+
	\lambda
	\|\mathbf{X}\|_1 ,
\end{equation}
where $\lambda$ is the regularization coefficient, set to 0.1.

\section{Experiment}\label{Section:Experiment}
\subsection{Experimental Setup}
\subsubsection{Dataset} We evaluate our method on four widely used IRSTD benchmarks: NUDT-SIRST \cite{DNANet}, IRSTD-1K \cite{ISNet}, SIRST \cite{ALCNet}, and SIRST-Aug \cite{AGPCNet}. NUDT-SIRST contains 1,327 images with a standard 50:50 split, IRSTD-1K includes 1,000 real-world images with an 80:20 split, and SIRST provides 427 images. SIRST-Aug further expands SIRST to 9,070 images through augmentation with 8,525 training and 545 testing samples. Following official protocols, all experiments use a fixed input resolution of \(256\times256\).

\subsubsection{Evaluation Metrics} For evaluation, we employ both pixel-level and target-level metrics. Specifically, intersection over union (IoU) and the \(F\)-measure (\(F_1\)) are used for pixel-level accuracy evaluation, while the probability of detection (\(P_d\)) and false alarm rate (\(F_a\)) are adopted to assess target-level detection performance.

\subsubsection{Implementation Details}  The network is implemented in PyTorch 1.8.2 and trained with Adam on an NVIDIA A100 GPU using a learning rate of $1\times10^{-4}$, poly decay, batch size 8, and 800 epochs. Images are normalized to [0,1].

\subsection{Comparison with State-of-the-Arts}
\begin{table*}[t!]
	\caption{Detection results achieved by different SOTA methods. The best and second results are in \first{red} and \textcolor{blue}{\textbf{blue}}, respectively. The metrics considered include mIoU ($10^{-2}$), $F_1$ ($10^{-2}$), $P_d$ ($10^{-2}$), $F_a$ ($10^{-5}$), Times ($256\times256$ with batchsize 1).}\label{Tab:SOTA}
	\centering
	\setlength{\tabcolsep}{2.5pt}
	\resizebox{\textwidth}{!}{\begin{tabular}{l|c|cccc|cccc|cccc|cccc|cc}
			\noalign{\hrule height 1pt}
			\multicolumn{1}{l|}{\multirow{2}*{Methods}} & \multicolumn{1}{c|}{\multirow{2}*{Venue}} & \multicolumn{4}{c|}{\textbf{IRDST-1K}} & \multicolumn{4}{c|}{\textbf{NUDT-SIRST}}
			& \multicolumn{4}{c|}{\textbf{SIRST}}
			& \multicolumn{4}{c|}{\textbf{SIRST-Aug}} & \multirow{1}*{Params}  & \multirow{1}*{Times(s)} \\
			\cline{3-18}
			\multicolumn{1}{l|}{~} &\multicolumn{1}{c|}{~} & IoU & $F_1$ & $P_d$ & $F_a$ & IoU & $F_1$ & $P_d$ & $F_a$ & IoU & $F_1$ & $P_d$ & $F_a$ & IoU & $F_1$ & $P_d$ & $F_a$ & (M) & 3080TI         \\ \noalign{\hrule height 1pt}
			ISNet \cite{ISNet} & CVPR'22 & 63.00 & 77.29 & 88.66 & 5.52 & 87.05 & 93.08 & 96.83 & 4.05 & 66.96 & 80.20 & 96.33 & 9.85 & 71.10 & 83.11 & 97.66 & 30.33 & \second{0.967} & 0.0270 \\
			DNANet \cite{DNANet} & TIP'23 & 62.23 & 76.72 & 92.44 & 5.64 & 87.73 & 93.47 &97.67 & 5.52 & 73.12 & 84.48 & \first{100.0} & 10.26 & 70.56 & 82.74 & 96.56 & 36.49 & 4.697 & 0.0437 \\
			UIUNet \cite{UIUNet} & TIP'23 & 63.08 & 77.36 & 92.10 & 6.09 & 88.94 & 94.15 & 95.23 & 1.53 & 72.15 & 83.82 & 98.16 & 7.90 & 70.76 & 82.88 & 96.70 & 34.30 & 50.54 & 0.0541 \\
			AGPCNet \cite{AGPCNet} & TAES'23 & 60.44 & 75.34 & 90.72 & 6.10 & 85.02 & 91.90 & 97.88 & 4.34 & 68.19 & 81.08 & \second{99.08} & 12.09 & 72.07 & 83.77 & 98.62 & 29.39 & 12.36 & 0.1173 \\
			IRSAM \cite{IRSAM} & ECCV'24 & 63.58 & 77.74 & 92.10 & 4.21 & 87.85 & 93.53 & \first{99.05} & 3.21 & 67.97 & 80.93 & \second{99.08} & 8.34 & 74.70 & 85.52 & 99.03 & 25.94 & 12.33 & 0.0131 \\
			FreqFusion \cite{chen2024frequency} & TPAMI'24 &55.00 &70.97 &94.50 &8.70 & 83.19 & 90.82 & 98.52 & 6.70 & 58.92 & 74.15 & 93.58 & 10.20 & 59.79 & 74.83 & 92.71 & 27.00 & 35.99 & 0.0209 \\
			MSHNet \cite{MSHNet} & CVPR'24 &64.83 &78.67 &92.78 & 7.14 &75.96 &86.34 & 95.23 &8.28 &64.64 & 78.52 & 98.16 &15.66 & 71.93 & 83.67 & 99.03 & 32.06 & 4.065 & 0.0483 \\
			DRPCANet \cite{TGRS2025DRPCANet} & TGRS'25 & 64.14 & 78.15 & 92.09 & 17.92 &94.16 & 96.99 & 98.41 & 2.553 & \second{75.52} & \second{86.05} & \second{99.08} & 3.730 & \first{76.50} & \first{86.69} & 98.62 & 30.60 & 1.169 & \second{0.0101}\\
			PConv \cite{yang2025pinwheel} & AAAI'25 &64.25 &78.23 & 93.81 & 6.10 & 81.69 & 89.92 & 98.62 & 6.60 & 58.90 & 74.14 & 94.50 & 9.60 & 70.93 & 82.99 & 98.21 & 36.20 & 4.064 & 0.0486 \\
			UKAN \cite{li2025u} & AAAI'25 & 56.61 & 72.30 & 92.10 & 7.90 & 80.23 & 89.03 & 98.31 & 7.60 & 62.82 & 77.17 & 96.33 & 9.00 & 66.25 & 79.70 & 94.50 & 26.80 & 9.384 & 0.0225\\
			RSFUNet \cite{11014512} & TNNLS'25 &54.64 & 70.67 & 92.44 & 8.40 & 92.35 & 96.02 & 98.73 & 2.30 & 61.13 & 75.87 & 95.41 & 11.40 & 72.24 & 83.88 & 98.07 & 36.20 & 1.678 & 0.1982 \\
			nnWNet \cite{zhou2025nnwnet} & CVPR'25 & 60.43 & 75.33 & \first{96.22} & 9.20 & 91.01 & 95.29 & 98.73 & 2.00 & 60.96 & 75.75 & 95.41 & 12.30 & 71.91 & 83.66 & 98.35 & 30.20 & 7.041 & 0.0223 \\
			SRSNet \cite{11174084} & TIP'25 & 54.95 & 70.93 & 90.03 & 5.90 & 91.76 & 95.70 & 99.05 & 2.50 & 58.38 & 73.72 & 95.41 & 11.60 & 67.84 & 80.84 & 94.64 & 27.20 & \first{0.733} & 0.1086 \\
			SENet \cite{10844040} & TIP'25 & 43.95 & 61.06 & 80.75 & \second{2.41}&57.96&73.38&91.00&4.31 &48.40 & 65.30&90.40&1.497 &70.43&82.65&99.17&3.281&152.6&0.0238\\
			FDConv \cite{chen2025frequency} & CVPR'25 & 57.47 & 73.00 & 93.13 & 6.20 & 91.94 & 95.80 & 96.30 & 1.70 & 57.27 & 72.83 & 96.33 & 10.80 & 70.57 & 82.75 & 98.76 & 28.00 & 1.476 & 0.0826 \\
			TherNet \cite{chen2025thernet} & TPAMI'25 & 49.15 & 65.90 & \second{95.19} & 14.39 &67.20 & 80.38 & 97.46 & 14.92 & 37.79 & 54.86 & 94.50 & 46.69 & 61.31 & 76.01 & 97.80 & 29.62 & 84.06 & 0.0870\\
			CSSAM \cite{11297835} &TPAMI'26 &49.82&66.51&90.24&8.29&57.72&73.19&95.13&2.72&75.98&86.35&98.10&\first{0.506}&73.98&85.04&99.59&\first{0.608}&230.0&0.0285\\
			RPCANet$^{++}$ \cite{11537388} & TPAMI'26 & \second{64.93} & \second{78.73} & 89.70 & 4.35 & \second{94.39} & \second{97.12} & 98.41 & 1.34 & 74.76 & 85.47 & \first{100.0} & 10.77 & 74.89 & 85.44 & 98.76 & 28.00 & 2.915 & 0.0471 \\
			GBNet \cite{11515007} & TIP'26 &46.65&63.62&91.40&2.665&76.74&86.64&98.09&\first{0.685}&60.37&75.29&95.43&\second{2.210}&72.86&84.30&\second{99.31}&\second{1.506}&77.10&0.0500\\
			DHiF \cite{11373245} & TCSVT'26 &60.25&75.20&90.38&3.609&78.31&87.83&96.51&4.858&65.27&78.99&94.30&4.267&69.87&82.26&\first{99.45}&18.36 &4.050&0.0485\\
			MCFNet \cite{11568953} & TIP'26 &63.38&77.58&89.00&\first{2.274}&91.39&95.50&98.31&\second{1.232}&65.75&79.33&97.34&6.400&75.28&85.90&99.17&5.412&10.98&0.0454\\
			NS-FPN \cite{yuan2026seeing} & CVPR'26 &61.29&76.00&89.69&3.844&80.65&89.29&96.93&3.645&64.89&78.71&95.06&6.273&74.66&85.49&99.17&8.553&4.332&0.0252\\\noalign{\hrule height 1pt}
			RFONet & Ours'26 & \first{72.14} & \first{83.82} & 92.26 & 4.18 &\first{95.23}&\first{97.56}&\second{98.94}&1.278&\first{77.99}&\first{87.63}&95.45&5.382&\second{75.93}&\second{86.32}&98.76&2.228&1.160&\first{0.0063}\\
			\noalign{\hrule height 1pt}
	\end{tabular}}
\end{table*}
\subsubsection{Quantitative Evaluation}
IRSTD-1K exhibits pronounced multi-scale target characteristics, with 18\% of targets occupying areas of (0, 10] pixels, 46\% falling within (10, 40] pixels, 27\% within (40, 100] pixels, and 9\% exceeding 100 pixels. SIRSTAUG mainly consists of typical small targets, where 97\% of targets have pixel areas within (0, 10]. NUDT-SIRST presents a broader cross-scale distribution, with 25\% of targets occupying (0, 10] pixels, 44\% within (10, 40] pixels, and 30\% within (40, 100] pixels, while containing few extremely large targets. In the SIRST dataset, 55\% of targets have pixel areas in the (0, 10] range, while the remaining targets are relatively evenly distributed across other scales.
\begin{prop}[Multi-scale Target Response Coverage of V-cycle]
	\label{prop:multi_scale_target}
	Let \(T_s\) denote the target component at spatial scale \(s\) within the
	ERF hierarchy, and let \(\Phi_k\) be the ERFBlock with scale \(k\). Define
	the response gain as
	$g_{k,s}=\frac{\|\Phi_k(T_s)\|}{\|T_s\|}$.
	Consistent with the scale-frequency correspondence established in Property~\ref{prop:ERF_frequency}, the scale-aligned operator provides the largest expected response:
	$\mathbb{E}[g_{s,s}]>
	\mathbb{E}[g_{k,s}],\quad k\neq s$.
	For a single-scale organization with fixed scale \(f\neq s\), $G_{\mathrm{single}}=\mathbb{E}[g_{f,s}]$. In contrast, the V-cycle trajectory \(\mathcal{S}_V=\{1,2,\dots,S,\dots,1\}\) covers all ERF scales. Hence,
	$G_{\mathrm{vc}}
	=
	\mathbb{E}
	\left[
	\max_{k\in\mathcal{S}_V}g_{k,s}
	\right]
	\geq
	\mathbb{E}[g_{s,s}]
	>
	G_{\mathrm{single}}$.
	Therefore, V-cycle guarantees that each target scale encounters its	scale-aligned ERF transformation, avoiding the scale limitation of fixed single-scale refinement.
\end{prop}
\begin{proof}
	Since \(s\in\mathcal{S}_V\), we have
	\begin{equation}
		\max_{k\in\mathcal{S}_V}g_{k,s}\geq g_{s,s}.
	\end{equation}
	Taking expectations and applying Property~\ref{prop:ERF_frequency} yields
	\begin{equation}
		G_{\mathrm{vc}}
		\geq
		\mathbb{E}[g_{s,s}]
		>
		\mathbb{E}[g_{f,s}]
		=
		G_{\mathrm{single}} .
	\end{equation}
	Thus, V-cycle provides a larger achievable response gain than any scale-mismatched single-scale organization.
\end{proof}
\noindent\textbf{Remark.} Proposition \ref{prop:multi_scale_target} theoretically indicates that the V-cycle-based ERF organization enables comprehensive coverage of multi-scale target responses through hierarchical receptive field refinement. 

\begin{figure}
	\centering
	\includegraphics[width=0.9\linewidth]{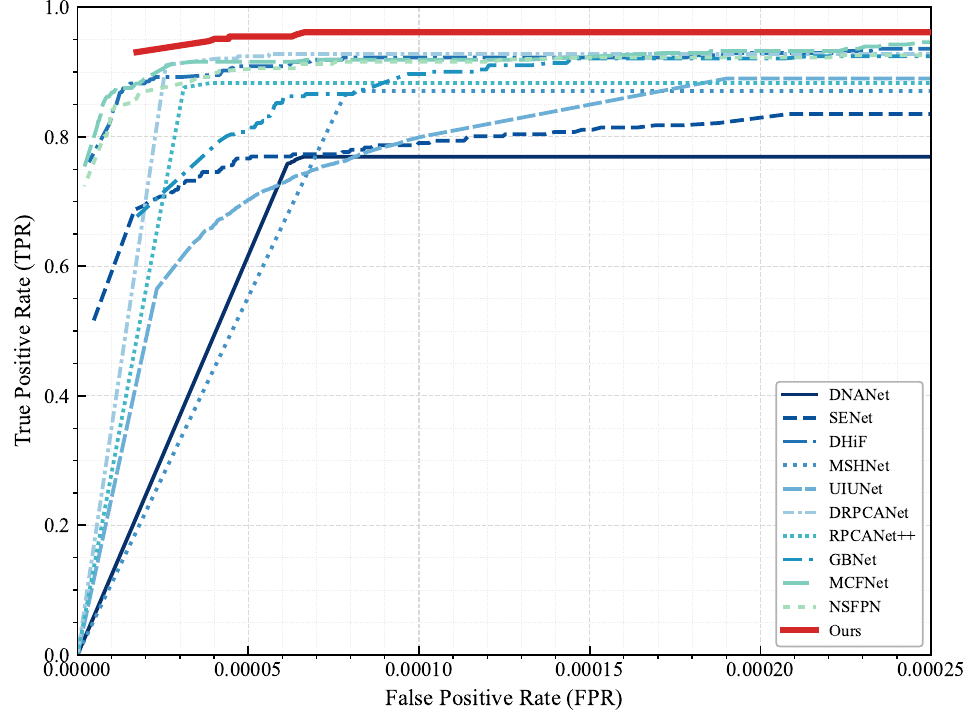}
	\caption{ROC curves of our RFONet and other approaches on IRSTD-1k.}
	\label{fig:ROC}   
\end{figure}
As reported in Tab. \ref{Tab:SOTA}, the proposed method consistently achieves superior performance on four multi-scale target benchmarks, despite requiring only 1.16M parameters. In particular, the remarkable performance on IRDST-1K, which contains the most challenging scale variations, further demonstrates the effectiveness of the proposed ERF scheduling strategy in capturing targets across different spatial scales. This empirical evidence is consistent with the theoretical analysis and highlights the advantage of V-cycle ERF organization in achieving parameter-efficient multi-scale representation learning. To further verify this, we present the ROC curves in Fig. \ref{fig:ROC}, evaluated on the IRSTD-1k dataset, where our approach achieves the best AUC score.

\subsubsection{Qualitative Evaluation}Infrared targets generally present diverse spatial scales and are often partially occluded in real-world scenarios. To investigate the robustness of the proposed ERF organization under such conditions, we first establish a theoretical analysis (Corollary~\ref{cor:occlusion}) showing that V-cycle refinement preserves scale-consistent target responses despite moderate occlusion. Qualitative comparisons are then provided to validate this property in challenging occluded scenes.
\begin{corollary}[Occlusion Robustness of V-cycle ERF Refinement]
	\label{cor:occlusion}
	Assume that the infrared imaging process before saturation follows the additive model
	\begin{equation}
		X=T+B+\epsilon,
	\end{equation}
	where \(T\), \(B\), and \(\epsilon\) denote the target response, background, and noise, respectively. Under partial occlusion, the observation becomes
	\begin{equation}
		X_o=M\odot T+B+\epsilon,
	\end{equation}
	where \(M\in[0,1]\) denotes the occlusion mask.	For compact targets, the dominant frequency characteristics are governed by	target scale. Moderate occlusion mainly suppresses response magnitude while	approximately maintaining the scale-frequency correspondence:
	\begin{equation}
		\mathbb{E}[g_{s,s}(T_o)]
		>
		\mathbb{E}[g_{k,s}(T_o)],
		\quad k\neq s,
	\end{equation}
	where \(T_o=M\odot T\).	According to Proposition~\ref{prop:multi_scale_target}, the V-cycle	trajectory \(\mathcal{S}_V=\{1,2,\dots,S,\dots,1\}\) contains the scale-aligned ERF operator for any target scale \(s\). Hence,
	\begin{equation}
		\mathbb{E}
		\left[
		\max_{k\in\mathcal{S}_V}g_{k,s}(T_o)
		\right]
		\geq
		\mathbb{E}[g_{s,s}(T_o)] .
	\end{equation}
	Therefore, V-cycle ERF refinement preserves scale-consistent target responses under partial occlusion by ensuring that the scale-aligned operator remains available, thereby avoiding degradation that would arise from a fixed mismatched scale.
\end{corollary}
\begin{figure*}
	\centering
	\includegraphics[width=\linewidth]{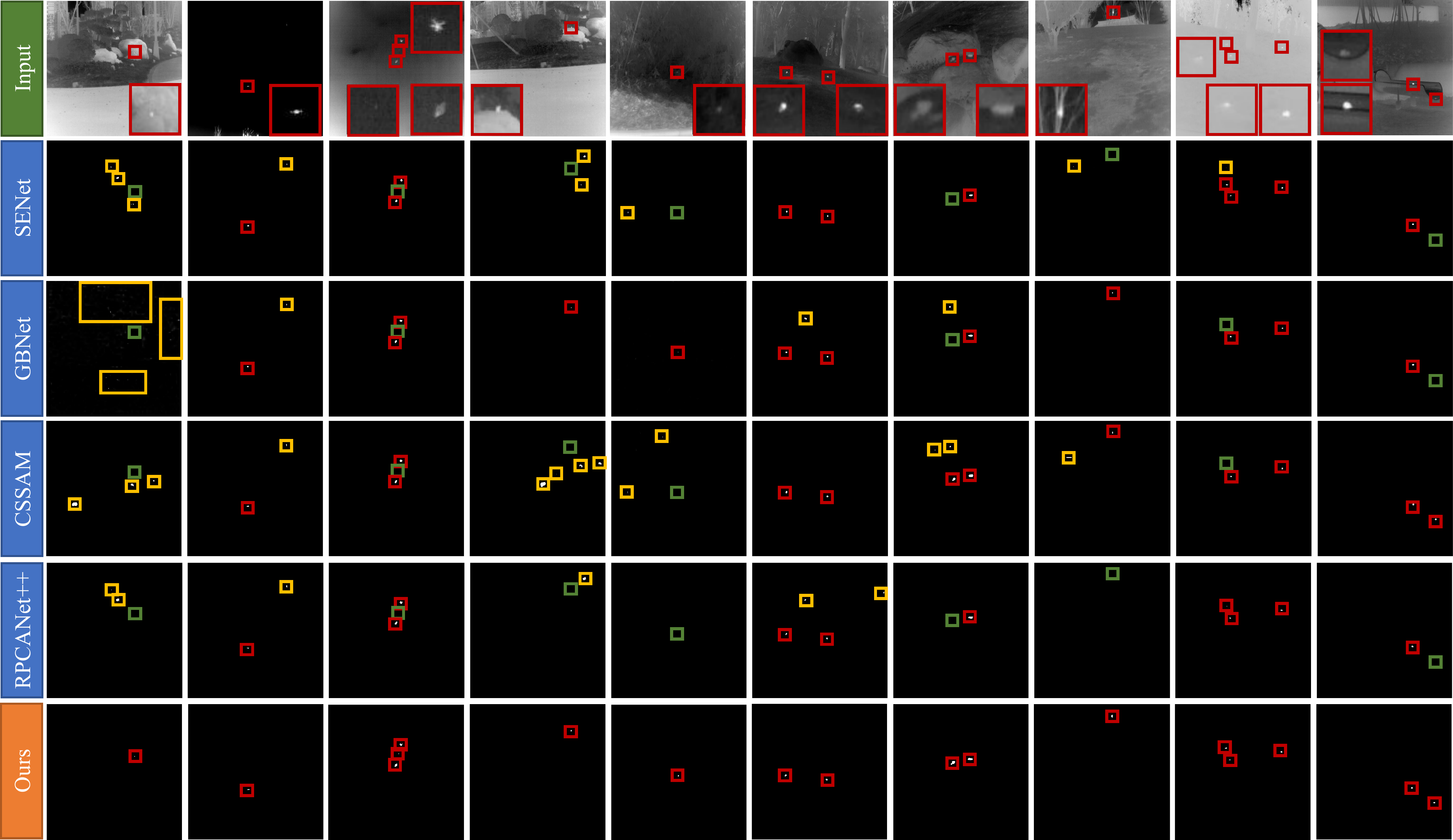}
	\caption{Visualization comparison of detection results between Ours and the other models across various challenging scenarios. Correct detections, false alarms, and missed detections are indicated by \textcolor{red}{red}, \textcolor[rgb]{1.0, 0.75, 0.0}{yellow}, and \textcolor[rgb]{0.33, 0.50, 0.20}{green} bounding boxes, respectively.}
	\label{fig:visual_more}   
\end{figure*}
As shown in Fig. \ref{fig:visual_more}, the proposed method demonstrates superior performance across various challenging scenarios. It effectively adapts to targets of varying sizes, signal-to-clutter ratios and occluded scenes.

\subsection{Analysis on Noise Tolerance}
To evaluate robustness under degraded imaging conditions, we conduct noise tolerance analysis. Considering that Gaussian-like noise increases with infrared detector sensitivity degradation, we adopt the NoisySIRST dataset \cite{SeRankDet} with different noise levels to simulate low-quality imaging scenarios. Furthermore, we introduce the Parameter Contribution Ratio (PCR), \(\mathrm{PCR}=\mathrm{mIoU}/\mathrm{Params}(M)\), to evaluate parameter efficiency by measuring the performance gain per parameters.
\begin{prop}[Perturbation Stability of V-cycle ERF Refinement]
	\label{prop:vcycle_stability}
	Consider a V-cycle ERF refinement sequence $\{r_1,r_2,\dots,r_K\}$, where the finest-scale operator	$r_f$ is repeatedly applied in a single-scale refinement scheme.
	Let $J_{r_k}$ and $J_f$ denote the Jacobian matrices of the corresponding ERF transformations. If
	$\|J_{r_k}\|_2\leq \|J_f\|_2,
	\quad k=1,\dots,K$,
	then the accumulated perturbation amplification of the V-cycle satisfies
	$\prod_{k=1}^{K}
	(1+\|J_{r_k}\|_2)
	\leq
	(1+\|J_f\|_2)^K $.
	This condition reflects that fine-scale operators have larger local sensitivity, while larger ERFs suppress perturbations through broader context aggregation. Hence, progressive ERF scaling mitigates the accumulation of local perturbations.
\end{prop}
\begin{proof}
	According to first-order perturbation analysis, the perturbation propagation through the refinement sequence is bounded by
	\begin{equation}
		\|\Delta E^{(K)}\|
		\leq
		\prod_{k=1}^{K}
		(1+\|J_{r_k}\|_2)
		\|\Delta E^{(0)}\|.
	\end{equation}
	Under the condition
	$\|J_{r_k}\|_2\leq\|J_f\|_2$,
	each refinement stage satisfies
	\begin{equation}
		1+\|J_{r_k}\|_2
		\leq
		1+\|J_f\|_2 .
	\end{equation}
	Taking the product over all refinement stages yields
	\begin{equation}
		\prod_{k=1}^{K}
		(1+\|J_{r_k}\|_2)
		\leq
		(1+\|J_f\|_2)^K .
	\end{equation}
	Hence, the V-cycle refinement produces a smaller worst-case perturbation amplification bound compared with repeatedly applying the finest-scale
	operator.
\end{proof}
\noindent\textbf{Remark.} Proposition \ref{prop:vcycle_stability} indicates that robustness does not originate from simply enlarging ERFs, but from organizing ERFs along the refinement trajectory. The V-cycle schedule progressively introduces contextual constraints, allowing structured residuals to be preserved while reducing inconsistent responses caused by random perturbations.
\begin{table}[t]
	\footnotesize
	\centering
	\caption{Analysis on Noise Tolerance}
	\label{tab:noise}
	\setlength{\tabcolsep}{4.3pt}
	\begin{tabular}{c|cc|cc|cc}
		\noalign{\hrule height 1pt}
		\multirow{2}{*}{Methods} & \multicolumn{2}{c|}{$\sigma_n=10$}  & \multicolumn{2}{c|}{$\sigma_n=20$}   & \multicolumn{2}{c}{$\sigma_n=30$} \\ \multicolumn{1}{l|}{} 
		&  IoU(\%)$\uparrow$ &  PCR$\uparrow$ & IoU(\%)$\uparrow$ &  PCR$\uparrow$ & IoU(\%)$\uparrow$ &  PCR$\uparrow$ \\
		\noalign{\hrule height 1pt}
		DHiF  & 64.99 & 16.04 & 64.12 & 15.83 & 63.57  & 15.69 \\
		GBNet  & 59.02 & 0.765 & 52.09 & 0.675 & 50.21 & 0.651\\
		MCFNet & \second{67.18} & 6.118 & 62.44 & 5.686 & 61.95 & 5.642\\
		NS-FPN  & 64.98 & \second{15.00} & \second{64.17} & \second{14.81} & \second{64.12}  & \second{14.80} \\
		RFONet & \first{74.17} & \first{63.93} & \first{73.33} & \first{63.21} & \first{69.13} &\first{59.59}\\
		\noalign{\hrule height 1pt}
	\end{tabular} \label{Tab:Noisy}
\end{table}

As reported in Tab. \ref{Tab:Noisy}, increasing noise intensity leads to apparent performance degradation for most competing methods, while the proposed method exhibits remarkable stability under severe noise conditions. Notably, the superior PCR performance indicates that the proposed model achieves enhanced robustness under a limited parameter budget. This empirical finding aligns well with Proposition \ref{prop:vcycle_stability}, which theoretically demonstrates that the V-cycle-based ERF scheduling strategy improves feature refinement stability by progressively integrating multi-scale contextual information and suppressing noise-induced perturbations. Consequently, the experimental results provide further evidence supporting the effectiveness of ERF ordering for improving robustness against noisy interference.
\subsection{Cross-Dataset Generalization Evaluation}
\begin{table}
	\centering
	\caption{Cross-Dataset Generalization Evaluation.}
	\setlength{\tabcolsep}{3.5mm}{
		\begin{tabular}{c|cc|cc}
			\noalign{\hrule height 1pt}
			\multirow{2}{*}{Methods} & \multicolumn{2}{c|}{\begin{tabular}[c]{@{}c@{}}IRDST-1K(Train)\\ NUDT-SIRST(Test)\end{tabular}} & \multicolumn{2}{c}{\begin{tabular}[c]{@{}c@{}}NUDT-SIRST(Train)\\ IRDST-1K(Test)\end{tabular}} \\ \cline{2-5} 
			&  IoU(\%)              & $P_d$(\%)             &   IoU(\%)             & $P_d$(\%)             \\ \noalign{\hrule height 1pt}
			DNANet & \second{50.4} & \second{84.8} & 42.0 & 85.2      \\
			ISNet & 44.3 & 75.9 & 25.9 & 86.5 \\
			UIUNet & 45.0 & 83.6 & \second{44.4} & \second{91.2} \\
			MSHNet & 46.3 & 77.1 & 40.5 & 89.6 \\
			RPCANet$^{++}$ & 21.2 & 64.8 & 16.8 & 75.6 \\
			RFONet &\first{61.3} &\first{91.9} & \first{54.4} &\first{93.3}\\\noalign{\hrule height 1pt}
		\end{tabular}
	\label{tab:Generalization_ability}
}
\end{table}
To evaluate cold-start capability, we conduct cross-dataset generalization experiments without additional training. This setting simulates practical scenarios where target-specific data are unavailable and evaluates whether the learned representations capture transferable target-background characteristics.
\begin{prop}[Frequency-shift Sensitivity of ERF Organization]
	\label{prop:generalization}
	Assume that the domain variation between two infrared datasets is reflected	by a bounded shift of the residual frequency distribution: $\|\Delta w\|_1\leq\delta$, where $\Delta w=w_{target}-w_{source}$.
	For a fixed initial residual decomposition $\{E_s\}_{s=1}^{S}$, the attenuation coefficient	$\Gamma_s(\mathcal{A})$ determines the remaining error of each frequency component after applying an ERF organization strategy $\mathcal{A}$. The sensitivity to the domain shift satisfies
	$\left|
	\Delta\mathcal{L}(\mathcal{A})
	\right|
	\leq
	\|\Delta w\|_1
	\max_s
	\left(
	\Gamma_s(\mathcal{A})\|E_s\|
	\right)$,
	where
	$\Gamma_s(\mathcal{A})
	=
	\prod_{k=1}^{K}
	(1-\eta_{s_k,s})$
	denotes the residual attenuation coefficient of frequency component $s$, and $\eta_{s_k,s}$ represents the correction strength of the $k$-th ERF operator for this frequency component.
	A single-scale ERF organization imposes a fixed frequency bias, leaving out-of-range residual components insufficiently suppressed. By progressively combining multiple ERF scales, the V-cycle organization achieves complementary frequency attenuation. Under the same initial residual decomposition, the worst-case residual sensitivity satisfies
	$\max_s
	\left(
	\Gamma_s(\mathcal{A}_{vc})\|E_s\|
	\right)
	\leq
	\max_s
	\left(
	\Gamma_s(\mathcal{A}_{single})\|E_s\|
	\right)$,
	where $\mathcal{A}_{single}$ denotes a repeated single-scale refinement. Therefore, the V-cycle ERF organization exhibits lower sensitivity to bounded frequency distribution shifts.
\end{prop}
\begin{proof}
	According to the residual refinement formulation, the remaining error after	applying an ERF organization strategy $\mathcal{A}$ can be decomposed into different frequency components:
	\begin{equation}
		\mathcal{L}(\mathcal{A},w)
		=
		\sum_s
		w_s
		\Gamma_s(\mathcal{A})
		\|E_s\|.
	\end{equation}
	When the residual frequency distribution changes from
	$w_{source}$ to
	$w_{target}=w_{source}+\Delta w$,
	the corresponding variation of the refinement error is
	\begin{equation}
		\begin{aligned}
			\Delta\mathcal{L}
			&=
			\mathcal{L}(\mathcal{A},w+\Delta w)
			-
			\mathcal{L}(\mathcal{A},w)
			\\
			&=
			\sum_s
			\Delta w_s
			\Gamma_s(\mathcal{A})
			\|E_s\|.
		\end{aligned}
	\end{equation}
	Taking the absolute value and applying Hölder's inequality gives
	\begin{equation}
		\begin{aligned}
			|\Delta\mathcal{L}|
			&\leq
			\sum_s
			|\Delta w_s|
			\Gamma_s(\mathcal{A})
			\|E_s\|
			\\
			&\leq
			\|\Delta w\|_1
			\max_s
			\left(
			\Gamma_s(\mathcal{A})
			\|E_s\|
			\right).
		\end{aligned}
	\end{equation}
	Under a fixed residual decomposition, the sensitivity is bounded by	$\max_s\Gamma_s(\mathcal{A})\|E_s\|$. Single-scale refinement preserves larger residual responses for mismatched
	frequencies, whereas V-cycle reduces the maximum term through multi-scale attenuation:
	\begin{equation}
		\max_s
		\left(
		\Gamma_s(\mathcal{A}_{vc})
		\|E_s\|
		\right)
		\leq
		\max_s
		\left(
		\Gamma_s(\mathcal{A}_{single})
		\|E_s\|
		\right).
	\end{equation}
	Therefore, under bounded frequency distribution shifts, the V-cycle ERF	organization achieves lower refinement sensitivity than a fixed single-scale organization.
\end{proof}
\noindent\textbf{Remark.}
The proposition \ref{prop:generalization} does not claim that V-cycle universally eliminates all frequency components. Instead, it reveals that organizing multiple ERF scales reduces the dependence on any specific frequency preference, thereby improving the stability under domain-dependent residual distribution changes.

As shown in Tab. \ref{tab:Generalization_ability}, our method achieves superior cross-dataset generalization. The improvement stems from the V-cycle ERF organization, which continuously integrates multi-scale receptive fields within each ERFBlock and enables adaptive residual refinement across scales. By contrast, existing methods usually introduce large ERFs only in isolated modules, lacking systematic scale transitions. Interestingly, RPCANet$^{++}$ \cite{11537388} preserves an almost invariant ERF size throughout the network, which limits its ability to capture diverse residual patterns. These empirical findings further validate Proposition \ref{prop:generalization}, which establishes that ERF organizations with adaptive multi-scale refinement achieve improved generalization by alleviating sensitivity to residual distribution variations.

\subsection{Ablation Study}
\begin{theorem}[Existence of Optimal Refinement Depth]
	\label{Thm:ERF_length}
	Consider a progressive ERF refinement process with refinement depth $K$. The accumulated refinement objective is formulated as
	$\mathcal{L}(K)
	=
	\|E^{(0)}\|
	-
	S_{\max}(1-\rho^K)
	+
	D_{\max}(1-e^{-\beta K})$,
	where $0<\rho<1$ characterizes the diminishing residual reduction brought by iterative refinement, and $\beta>0$ controls the saturation rate of the representation distortion introduced by excessive refinement.
	If
	$\beta<-\ln\rho$,
	then there exists a finite optimal refinement depth $K^*$ minimizing
	$\mathcal{L}(K)$.
\end{theorem}
\begin{proof}
	The derivative of $\mathcal{L}(K)$ with respect to the refinement depth is
	\begin{equation}
		\frac{\mathrm{d}\mathcal{L}(K)}{\mathrm{d}K}
		=
		S_{\max}\ln\rho\cdot\rho^K
		+
		D_{\max}\beta e^{-\beta K}.
	\end{equation}
	Since $0<\rho<1$, we have $\ln\rho<0$. Let
	$A=S_{\max}|\ln\rho|,
	\;
	B=D_{\max}\beta$.
	Then the derivative can be rewritten as
	\begin{equation}
		\frac{\mathrm{d}\mathcal{L}(K)}{\mathrm{d}K}
		=
		-A\rho^K
		+
		Be^{-\beta K}.
	\end{equation}
	The stationary point satisfies
	$Be^{-\beta K}
	=
	A\rho^K $.
	Using $\rho^K=e^{K\ln\rho}$, we obtain
	$e^{-K(\beta+\ln\rho)}
	=
	\frac{A}{B}$,
	which gives
	\begin{equation}
		K^*
		=
		\frac{\ln(B/A)}
		{\beta+\ln\rho}.
	\end{equation}
	A finite positive solution requires
	\(
	\beta+\ln\rho<0
	\),
	\ie,
	\(
	\beta<-\ln\rho
	\).
	Therefore, $\rho^K$ vanishes faster than $e^{-\beta K}$, making	$\mathcal{L}'(K)$ initially negative and eventually positive. Hence, $K^*$ is the unique minimizer. The second derivative at $K^*$ is
	\begin{equation}
		\mathcal{L}''(K^*)
		=
		Be^{-\beta K^*}
		(-\ln\rho-\beta).
	\end{equation}
	According to the condition $\beta<-\ln\rho$, we have
	$-\ln\rho-\beta>0$,
	and therefore
	\begin{equation}
		\mathcal{L}''(K^*)>0.
	\end{equation}
	This confirms that the stationary point corresponds to a local minimum,	which completes the proof.
\end{proof}
\noindent\textbf{Remark.}
Theorem~\ref{Thm:ERF_length} does not imply that deeper refinement always introduces degradation. Instead, it reveals a trade-off between diminishing residual reduction and saturated representation distortion. Therefore, the optimal number of ERF refinement stages is inherently task-dependent rather than universally fixed. Based on this observation, our ablation studies investigate both the effectiveness of the proposed V-cycle-inspired ERF scheduling and the selection of appropriate refinement depths and related hyperparameters.
\subsubsection{The Effectiveness of The Proposed ERFBlock}
\begin{table}[]
	\centering
	\caption{The Effectiveness of The Proposed ERFBlock}
	\setlength{\tabcolsep}{1.6pt}
	\begin{tabular}{cc|cccc}
		\noalign{\hrule height 1pt}
		\multicolumn{1}{c|}{\multirow{2}{*}{Module Variants}} & \multirow{1}*{Params} & \multicolumn{4}{c}{IRSTD-1K}  \\
		\multicolumn{1}{c|}{} &\multicolumn{1}{c|}{(M)}                                      & IoU(\%) $\uparrow$   & $F_1\text{(\%)}\uparrow$ & $P_d$(\%)$\uparrow$   & $F_a$($\times10^{-5}$) $\downarrow$    \\ \noalign{\hrule height 1pt}
		\multicolumn{1}{c|}{Vanilla ResUNet} & 31.4 & 64.56 & 78.44 & 88.66 & 2.03 \\ 
		\multicolumn{1}{c|}{w. ERFBlock}  & 1.16&72.14 & 83.82 & 92.26 & 4.18\\ 
		\noalign{\hrule height 1pt}
	\end{tabular} \label{Tab:ERFBlock}
\end{table}
\begin{figure*}
	\centering
	\includegraphics[width=\linewidth]{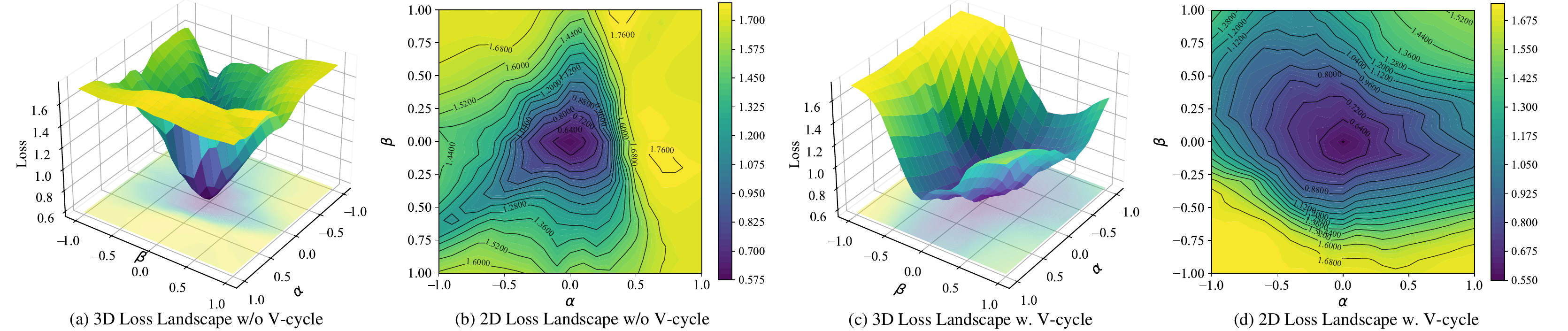}
	\caption{Loss landscape \cite{visualloss} visualization before and after incorporating the V-cycle scheduling. The proposed V-cycle strategy transforms the loss landscape from a rugged, high-curvature surface to a flatter and more structured topology, indicating improved optimization stability and generalization. The smoother gradients also facilitate more consistent convergence across different initializations.}
	\label{fig:loss_landscape}   
\end{figure*}
As shown in Tab. \ref{Tab:ERFBlock} and Fig. \ref{fig:loss_landscape}, replacing the basic building blocks of the vanilla U-Net with the proposed ERFBlocks consistently improves the detection performance, demonstrating the effectiveness of the proposed ERFBlock in enhancing feature refinement.
\subsubsection{The Effectiveness of V-cycle Strategy}
\begin{table}[]
	\centering
	\caption{The Effectiveness of V-cycle Strategy}
	\setlength{\tabcolsep}{4.3pt}
	\begin{tabular}{ccccc}
		\noalign{\hrule height 1pt}
		\multicolumn{1}{c|}{\multirow{2}{*}{Module Variants}} & \multicolumn{4}{c}{IRSTD-1K}  \\
		\multicolumn{1}{c|}{}                                      & IoU(\%) $\uparrow$   & $F_1\text{(\%)}\uparrow$ & $P_d$(\%)$\uparrow$   & $F_a$($\times10^{-5}$) $\downarrow$    \\ \noalign{\hrule height 1pt}
		\multicolumn{1}{c|}{Small $\rightarrow$ Large}  & 68.17 & 81.07 & 92.29 & 12.4 \\ 
		\multicolumn{1}{c|}{Large $\rightarrow$ Small}  & 67.15 & 80.35 & 92.59 & 11.9 \\ 
		\multicolumn{1}{c|}{V-cycle Strategy}  &72.14 & 83.82 & 92.26 & 4.18 \\
		\noalign{\hrule height 1pt}
	\end{tabular} \label{Tab:V-cycle}
\end{table}
To verify the effectiveness of the proposed V-cycle strategy, we construct two variants with monotonic ERF schedules, namely increasing-order and decreasing-order ERF organizations. As shown in Tab. \ref{Tab:V-cycle}, these results demonstrate that the performance gain does not originate from simply introducing multi-scale receptive fields, but from explicitly organizing ERF transitions during progressive refinement. This result demonstrates that a single-direction scale transition is insufficient to fully exploit the potential of multi-scale ERF refinement. Specifically, both increasing-order and decreasing-order strategies can be regarded as incomplete V-cycle trajectories, where the ERF evolution is restricted to only one direction and lacks complementary refinement across different spatial scales. In contrast, the V-cycle organization enables bidirectional scale traversal, allowing residual components to interact with multiple ERF scales during progressive refinement. This observation is consistent with Proposition~\ref{prop:multi_scale_target}, where the complete scale coverage of V-cycle provides more robust responses for targets with diverse spatial characteristics.
\subsubsection{The Effectiveness of Residual Learning}
\begin{table}[]
	\centering
	\caption{The Effectiveness of Residual Learning}
	\setlength{\tabcolsep}{4.3pt}
	\begin{tabular}{ccccc}
		\noalign{\hrule height 1pt}
		\multicolumn{1}{c|}{\multirow{2}{*}{Module Variants}} & \multicolumn{4}{c}{IRSTD-1K}  \\
		\multicolumn{1}{c|}{}                                      & IoU(\%) $\uparrow$   & $F_1\text{(\%)}\uparrow$ & $P_d$(\%)$\uparrow$   & $F_a$($\times10^{-5}$) $\downarrow$    \\ \noalign{\hrule height 1pt}
		\multicolumn{1}{c|}{w/o Residual}  & 64.29 & 78.26 & 91.92 & 17.1 \\ 
		\multicolumn{1}{c|}{w. Residual}  &72.14 & 83.82 & 92.26 & 4.18 \\ 
		\noalign{\hrule height 1pt}
	\end{tabular} \label{Tab:Residual}
\end{table}
Tab. \ref{Tab:Residual} reports the performance comparison with and without residual learning. Removing the residual connection leads to a significant performance degradation, demonstrating the importance of residual refinement in ERFBlock. This degradation can be attributed to the fact that directly stacking ERF transformations without identity preservation may progressively attenuate high-frequency target details, especially when large ERFs are involved, due to their stronger contextual aggregation capability. In contrast, the residual formulation preserves the original target representation while introducing multi-scale contextual refinement, thereby preventing excessive information smoothing during ERF transitions. This result also supports the theoretical analysis that residual learning maintains target response during multi-scale ERF refinement by providing a stable information-preserving pathway.
\subsubsection{Adaptive Receptive Field versus ERF Organization}
To further analyze the difference between adaptive receptive field learning and the proposed ERF organization, we formulate their adaptation spaces.
\begin{table}[]
	\centering
	\caption{Adaptive Receptive Field versus ERF Organization}
	\setlength{\tabcolsep}{4.0pt}
	\begin{tabular}{ccccc}
		\noalign{\hrule height 1pt}
		\multicolumn{1}{c|}{\multirow{2}{*}{Module Variants}} & \multicolumn{4}{c}{IRSTD-1K}  \\
		\multicolumn{1}{c|}{}                                      & IoU(\%) $\uparrow$   & $F_1\text{(\%)}\uparrow$ & $P_d$(\%)$\uparrow$   & $F_a$($\times10^{-5}$) $\downarrow$    \\ \noalign{\hrule height 1pt}
		\multicolumn{1}{c|}{w. RFAConv} & 63.00 &77.30&94.09&45.2\\
		\multicolumn{1}{c|}{w. DCNv4}  &65.83 &79.40&90.91&21.4 \\
		\noalign{\hrule height 1pt} 
		\multicolumn{1}{c|}{w. V-cycle}  &72.14 & 83.82 & 92.26 & 4.18 \\ 
		\multicolumn{1}{c|}{V-cycle + DCNv4}  & 73.21& 84.53 & 95.66 & 1.83 \\ 
		\noalign{\hrule height 1pt}
	\end{tabular} \label{Tab:DCN}
\end{table}
As shown in Tab.~\ref{Tab:DCN}, the proposed V-cycle based receptive field organization consistently surpasses the variant that abandons the scheduling and replaces all convolutional operations with DCNv4~\cite{xiong2024efficient} or RFAConv~\cite{ZHANG2026113208}. This comparison demonstrates that adaptive receptive field learning alone, without explicit multi-scale ordering, is insufficient for effective residual refinement. To further probe the relationship between the two paradigms, we integrate DCNv4 into the V-cycle trajectory of RFONet. The resulting model yields clear performance gains over the plain RFONet, confirming that adaptive spatial sampling and ERF scheduling are not competing but complementary. Specifically, the V-cycle furnishes a structured, frequency-aware refinement pathway, while DCNv4 provides flexible geometric adjustments along that pathway. Their combination leads to superior detection accuracy, validating the orthogonal nature of receptive field organization and operator-level adaptivity.
\subsubsection{Hyperparameters Discussion}
\begin{table}[]
	\centering
	\setlength{\tabcolsep}{4.3pt}
	\caption{Hyperparameters Discussion}
	\begin{tabular}{ccccc}
			\noalign{\hrule height 1pt}
			\multicolumn{1}{c|}{\multirow{2}{*}{Hyperparameters}} & \multicolumn{4}{c}{IRSTD-1K}  \\
			\multicolumn{1}{c|}{}                                     & IoU(\%) $\uparrow$   & $F_1\text{(\%)}\uparrow$ & $P_d$(\%)$\uparrow$   & $F_a$($\times10^{-5}$) $\downarrow$   \\ \noalign{\hrule height 1pt}
			\multicolumn{1}{c|}{$C = 8$} & 67.52 & 80.61 & 90.91 & 14.7 \\ 
			\multicolumn{1}{c|}{$C = 16$} &72.14 & 83.82 & 92.26 & 4.18\\
			\multicolumn{1}{c|}{$C = 32$} & 66.82 & 80.11 & 92.93 & 11.1 \\
			\hline
			\multicolumn{1}{c|}{L = 1} &68.60& 81.38 & 91.58 & 11.7 \\ 
			\multicolumn{1}{c|}{L = 2} &72.14 & 83.82 & 92.26 & 4.18 \\
			\multicolumn{1}{c|}{L = 3} & 67.12 & 80.32 & 92.26 & 15.4 \\
			\hline
			\multicolumn{1}{c|}{Stage = 2} & 64.94 & 78.74 & 91.92 & 17.6 \\
			\multicolumn{1}{c|}{Stage = 3} &72.14 & 83.82 & 92.26 & 4.18 \\
			\multicolumn{1}{c|}{Stage = 4} & 68.45 & 81.27 & 91.25 & 13.6 \\
			\hline
			\multicolumn{1}{c|}{V-cycle (s = 1)} &70.25 & 82.52 & 93.60 & 8.12 \\
			\multicolumn{1}{c|}{V-cycle (s = 2)} &72.14 & 83.82 & 92.26 & 4.18 \\
			\multicolumn{1}{c|}{V-cycle (s = 3)} & 71.62 & 83.46 & 91.58 & 7.96 \\
			\noalign{\hrule height 1pt}
	\end{tabular} \label{Tab:Size}
\end{table}
Finally, we investigate the influence of several key hyperparameters, including the number of stages, the base channel dimension (C), the number of ERFBlocks per stage (L), and the V-cycle scale parameter (s) in Eq. (\ref{Eq:s}). Following the standard U-Net configuration, we maintain the channel expansion and block allocation strategies as ($C_i=2C_{i-1}$) and ($L_i=L_{i-1}$), respectively. The corresponding ablation results are reported in Tab. \ref{Tab:Size}.

\subsection{Limitations}
\begin{theorem}[Task-dependent Optimal ERF Organization]
	\label{Thm:task_ERF}
	Let a task $\mathcal D_i$ contain residual components with different frequency distributions:
	$\mathbf E_i
	=
	\sum_s \mathbf E_{i,s}$,
	where $\mathbf E_{i,s}$ denotes the residual component associated with frequency band $\Omega_s$.
	Given an ERF organization strategy $\mathcal A=\{s_1,\cdots,s_K\}$, the residual after refinement can be characterized as
	$\mathcal J_i(\mathcal A)
	=
	\sum_s
	w_{i,s}
	\Gamma_s(\mathcal A)
	\|\mathbf E_{i,s}\|$,
	where $w_{i,s}$ denotes the importance of frequency component $s$ in task $\mathcal D_i$, and
	$\Gamma_s(\mathcal A)
	=
	\prod_{k=1}^{K}
	(1-\eta_{s_k,s})$
	represents the residual attenuation induced by the ERF organization. If two tasks $\mathcal D_a$ and $\mathcal D_b$ possess different residual frequency distributions, \ie,
	$w_{a,s}\neq w_{b,s}$,
	and different ERF organizations provide non-identical attenuation patterns, then their optimal ERF organizations are generally task-dependent:
	$\mathcal A_a^*
	=
	\arg\min_{\mathcal A}\mathcal J_a(\mathcal A),
	\;
	\mathcal A_b^*
	=
	\arg\min_{\mathcal A}\mathcal J_b(\mathcal A)$,
	and
	$\mathcal A_a^*
	\neq
	\mathcal A_b^*$
	in general.	Therefore, there is no universally optimal ERF organization independent of
	task characteristics.
\end{theorem}
\begin{proof}
	For task $\mathcal D_i$, the optimal ERF organization is obtained by
	\begin{equation}
		\mathcal A_i^*
		=
		\arg\min_{\mathcal A}
		\sum_s
		w_{i,s}
		\Gamma_s(\mathcal A)
		\|\mathbf E_{i,s}\|.
	\end{equation}
	Consider two tasks $\mathcal D_a$ and $\mathcal D_b$. Their optimization objectives are
	\begin{equation}
		\mathcal J_a(\mathcal A)
		=
		\sum_s
		w_{a,s}
		\Gamma_s(\mathcal A)
		\|\mathbf E_{a,s}\|,
	\end{equation}
	and
	\begin{equation}
		\mathcal J_b(\mathcal A)
		=
		\sum_s
		w_{b,s}
		\Gamma_s(\mathcal A)
		\|\mathbf E_{b,s}\|.
	\end{equation}
	Because different ERF organizations generate different frequency-dependent attenuation:
	\begin{equation}
		\Gamma_s(\mathcal A_p)
		\neq
		\Gamma_s(\mathcal A_q),
	\end{equation}
	the relative advantage of an ERF organization depends on the weights of	different residual components. Assume a universal optimal organization $\mathcal A^*$ exists for both tasks.
	Then $\mathcal A^*$ must simultaneously minimize
	\begin{equation}
		\sum_s
		w_{a,s}
		\Gamma_s(\mathcal A)
		\|\mathbf E_{a,s}\|
	\end{equation}
	and
	\begin{equation}
		\sum_s
		w_{b,s}
		\Gamma_s(\mathcal A)
		\|\mathbf E_{b,s}\|.
	\end{equation}
	However, different tasks impose distinct preferences on residual frequency components, leading to task-dependent optimal ERF trade-offs. Therefore, the optimal ERF organization is inherently coupled with the frequency composition of task-specific residual errors.
\end{proof}
\noindent\textbf{Remark.}
Theorem~\ref{Thm:task_ERF} reveals an important limitation of the current work. Although we demonstrate that ERF ordering plays a critical role in progressive residual refinement, this work does not attempt to identify a universally optimal ERF organization. Instead, the theorem indicates that the optimal ERF scheduling is inherently coupled with the residual frequency
distribution of each task. Therefore, incorporating ERF organization into neural architecture search (NAS) \cite{11230099} provides a promising direction for automatically discovering task-adaptive refinement strategies \cite{11568692}. However, searching over ERF scales, ordering patterns, repetition numbers, and diverse implementation forms (\eg, convolution-based, attention-based, or adaptive receptive field modulation mechanisms) constitutes a high-dimensional combinatorial optimization problem, which is generally NP-hard \cite{11599663}. Efficiently exploring this task-dependent ERF design space remains an open problem.
\section{Conclusion}\label{Section:Conclusion}
In this work, we revisit infrared small target detection from the perspective of effective receptive field (ERF) organization. We demonstrate that ERF scheduling is an independent architectural factor beyond receptive field enlargement, and reveal its frequency-selective and non-commutative properties through theoretical analysis. Based on these insights, we propose RFONet with V-cycle ERF scheduling to validate the effectiveness of organized receptive field refinement. Our analysis further indicates that optimal ERF organization is task-dependent, opening a new direction for adaptive architecture design.

\bibliographystyle{IEEEtran}
%\bibliography{reference}

\begin{thebibliography}{10}
	\providecommand{\url}[1]{#1}
	\csname url@samestyle\endcsname
	\providecommand{\newblock}{\relax}
	\providecommand{\bibinfo}[2]{#2}
	\providecommand{\BIBentrySTDinterwordspacing}{\spaceskip=0pt\relax}
	\providecommand{\BIBentryALTinterwordstretchfactor}{4}
	\providecommand{\BIBentryALTinterwordspacing}{\spaceskip=\fontdimen2\font plus
		\BIBentryALTinterwordstretchfactor\fontdimen3\font minus
		\fontdimen4\font\relax}
	\providecommand{\BIBforeignlanguage}[2]{{%
			\expandafter\ifx\csname l@#1\endcsname\relax
			\typeout{** WARNING: IEEEtran.bst: No hyphenation pattern has been}%
			\typeout{** loaded for the language `#1'. Using the pattern for}%
			\typeout{** the default language instead.}%
			\else
			\language=\csname l@#1\endcsname
			\fi
			#2}}
	\providecommand{\BIBdecl}{\relax}
	\BIBdecl
	
	\bibitem{10747828}
	J.~Li, X.~Wang, H.~Zhao, and Y.~Zhong, ``Learning a cross-modality anomaly
	detector for remote sensing imagery,'' \emph{IEEE Transactions on Image
		Processing}, vol.~33, pp. 6607--6621, 2024.
	
	\bibitem{10091756}
	Y.~Chen, Z.~Wang, and X.~Bai, ``Fuzzy sparse subspace clustering for infrared
	image segmentation,'' \emph{IEEE Transactions on Image Processing}, vol.~32,
	pp. 2132--2146, 2023.
	
	\bibitem{9217948}
	H.~Zhu, H.~Ni, S.~Liu, G.~Xu, and L.~Deng, ``Tnlrs: Target-aware non-local
	low-rank modeling with saliency filtering regularization for infrared small
	target detection,'' \emph{IEEE Transactions on Image Processing}, vol.~29,
	pp. 9546--9558, 2020.
	
	\bibitem{luo2016understanding}
	W.~Luo, Y.~Li, R.~Urtasun, and R.~Zemel, ``Understanding the effective
	receptive field in deep convolutional neural networks,'' \emph{Advances in
		neural information processing systems}, vol.~29, 2016.
	
	\bibitem{11483348}
	J.~Jia, Y.~J. Lee, and I.~Ojeda-Ruiz, ``An efficient k-way constrained
	normalized cut and its connection to algebraic multigrid method,'' \emph{IEEE
		Transactions on Image Processing}, vol.~35, pp. 4453--4466, 2026.
	
	\bibitem{11479915}
	R.~Li, W.~An, Y.~Wang, X.~Ying, Y.~Dai, L.~Wang, M.~Li, Y.~Guo, and L.~Liu,
	``Probing deep into temporal profile makes the infrared small target detector
	much better,'' \emph{IEEE Transactions on Pattern Analysis and Machine
		Intelligence}, vol.~48, no.~8, pp. 10\,157--10\,175, 2026.
	
	\bibitem{DNANet}
	B.~Li, C.~Xiao, L.~Wang, Y.~Wang, Z.~Lin, M.~Li, W.~An, and Y.~Guo, ``Dense
	nested attention network for infrared small target detection,'' \emph{IEEE
		Transactions on Image Processing}, vol.~32, pp. 1745--1758, 2023.
	
	\bibitem{UIUNet}
	X.~Wu, D.~Hong, and J.~Chanussot, ``Uiu-net: U-net in u-net for infrared small
	object detection,'' \emph{IEEE Transactions on Image Processing}, vol.~32,
	pp. 364--376, 2023.
	
	\bibitem{ISNet}
	M.~Zhang, R.~Zhang, Y.~Yang, H.~Bai, J.~Zhang, and J.~Guo, ``Isnet: Shape
	matters for infrared small target detection,'' in \emph{2022 IEEE/CVF
		Conference on Computer Vision and Pattern Recognition (CVPR)}, 2022, pp.
	867--876.
	
	\bibitem{CSRNet}
	F.~Lin, K.~Bao, Y.~Li, D.~Zeng, and S.~Ge, ``Learning contrast-enhanced
	shape-biased representations for infrared small target detection,''
	\emph{IEEE Transactions on Image Processing}, vol.~33, pp. 3047--3058, 2024.
	
	\bibitem{11174084}
	B.~Nian, F.~Tang, J.~Ding, J.~Yang, Z.~Zheng, S.~Kevin~Zhou, and W.~Liu, ``Srs:
	Siamese reconstruction-segmentation network based on dynamic-parameter
	convolution,'' \emph{IEEE Transactions on Image Processing}, vol.~34, pp.
	6318--6330, 2025.
	
	\bibitem{liu2023infrared}
	F.~Liu, C.~Gao, F.~Chen, D.~Meng, W.~Zuo, and X.~Gao, ``Infrared small and dim
	target detection with transformer under complex backgrounds,'' \emph{IEEE
		Transactions on Image Processing}, vol.~32, pp. 5921--5932, 2023.
	
	\bibitem{11568953}
	X.~Zhu, F.~Qin, C.~Wang, J.~Fan, F.~Lin, J.~Bai, C.~Zhang, and D.~Zhang,
	``Progressive fusion of multi-scale mamba context and local detail priors for
	infrared small target detection,'' \emph{IEEE Transactions on Image
		Processing}, vol.~35, pp. 6557--6571, 2026.
	
	\bibitem{11278553}
	M.~Zhang, X.~Li, J.~Guo, Y.~Li, and X.~Gao, ``Wmrnet: Wavelet mamba with
	reversible structure for infrared small target detection,'' \emph{IEEE
		Transactions on Image Processing}, vol.~34, pp. 8229--8242, 2025.
	
	\bibitem{11080263}
	B.~Yang, F.~Li, S.~Zhao, W.~Wang, J.~Luo, H.~Pu, M.~Zhou, and Y.~Pi, ``Mtmlnet:
	Multi-task mutual learning network for infrared small target detection and
	segmentation,'' \emph{IEEE Transactions on Image Processing}, vol.~34, pp.
	4414--4425, 2025.
	
	\bibitem{IRSAM}
	M.~Zhang, Y.~Wang, J.~Guo, Y.~Li, X.~Gao, and J.~Zhang, ``Irsam: Advancing
	segment anything model for infrared small target detection,'' in
	\emph{Computer Vision -- ECCV 2024}, pp. 233--249.
	
	\bibitem{11537388}
	F.~Wu, Y.~Dai, T.~Zhang, Y.~Ding, J.~Yang, M.-M. Cheng, and Z.~Peng,
	``Rpcanet$^{++}$: Deep interpretable robust pca for sparse object
	segmentation,'' \emph{IEEE Transactions on Pattern Analysis and Machine
		Intelligence}, pp. 1--18, 2026.
	
	\bibitem{11440142}
	Z.~Liang, Y.~Wang, L.~Wang, J.~Yang, Y.~Guo, L.~Liu, S.~Zhou, and W.~An,
	``Diving into epipolar transformers for light field super-resolution and
	disparity estimation,'' \emph{IEEE Transactions on Pattern Analysis and
		Machine Intelligence}, vol.~48, no.~7, pp. 8726--8743, 2026.
	
	\bibitem{10962317}
	Z.~Jin, Y.~Qiu, K.~Zhang, H.~Li, and W.~Luo, ``Mb-taylorformer v2: Improved
	multi-branch linear transformer expanded by taylor formula for image
	restoration,'' \emph{IEEE Transactions on Pattern Analysis and Machine
		Intelligence}, vol.~47, no.~7, pp. 5990--6005, 2025.
	
	\bibitem{9186840}
	S.~Li, Y.~A. Farha, Y.~Liu, M.-M. Cheng, and J.~Gall, ``Ms-tcn++: Multi-stage
	temporal convolutional network for action segmentation,'' \emph{IEEE
		Transactions on Pattern Analysis and Machine Intelligence}, vol.~45, no.~6,
	pp. 6647--6658, 2023.
	
	\bibitem{10872821}
	Y.~Chen, X.~Yuan, J.~Wang, R.~Wu, X.~Li, Q.~Hou, and M.-M. Cheng, ``Yolo-ms:
	Rethinking multi-scale representation learning for real-time object
	detection,'' \emph{IEEE Transactions on Pattern Analysis and Machine
		Intelligence}, vol.~47, no.~6, pp. 4240--4252, 2025.
	
	\bibitem{wu2025lrformer}
	Y.-H. Wu, S.-C. Zhang, Y.~Liu, L.~Zhang, X.~Zhan, D.~Zhou, J.~Feng, M.-M.
	Cheng, and L.~Zhen, ``Low-resolution self-attention for semantic
	segmentation,'' \emph{IEEE Transactions on Pattern Analysis and Machine
		Intelligence}, 2025.
	
	\bibitem{11130640}
	Y.~Zhang, X.~Ding, and X.~Yue, ``Scaling up your kernels: Large kernel design
	in convnets toward universal representations,'' \emph{IEEE Transactions on
		Pattern Analysis and Machine Intelligence}, vol.~47, no.~12, pp.
	11\,692--11\,707, 2025.
	
	\bibitem{mamba2}
	T.~Dao and A.~Gu, ``Transformers are {SSM}s: Generalized models and efficient
	algorithms through structured state space duality,'' in \emph{International
		Conference on Machine Learning (ICML)}, 2024.
	
	\bibitem{11304568}
	G.~Wu, J.~Jiang, K.~Jiang, X.~Liu, and L.~Nie, ``Dswinir: Rethinking
	window-based attention for image restoration,'' \emph{IEEE Transactions on
		Pattern Analysis and Machine Intelligence}, vol.~48, no.~4, pp. 4350--4366,
	2026.
	
	\bibitem{hou2024conv2former}
	Q.~Hou, C.-Z. Lu, M.-M. Cheng, and J.~Feng, ``Conv2former: A simple
	transformer-style convnet for visual recognition,'' \emph{IEEE Transactions
		on Pattern Analysis and Machine Intelligence}, 2024.
	
	\bibitem{xiong2024efficient}
	Y.~Xiong, Z.~Li, Y.~Chen, F.~Wang, X.~Zhu, J.~Luo, W.~Wang, T.~Lu, H.~Li,
	Y.~Qiao \emph{et~al.}, ``Efficient deformable convnets: Rethinking dynamic
	and sparse operator for vision applications,'' in \emph{Proceedings of the
		IEEE/CVF conference on computer vision and pattern recognition}, 2024, pp.
	5652--5661.
	
	\bibitem{11498646}
	Y.~Kim, S.~Han, and N.~Kwak, ``From local to global to mechanistic: An
	ierf-centered unified framework for interpreting vision models,'' \emph{IEEE
		Transactions on Pattern Analysis and Machine Intelligence}, pp. 1--17, 2026.
	
	\bibitem{gao2013infrared}
	C.~Gao, D.~Meng, Y.~Yang, Y.~Wang, X.~Zhou, and A.~G. Hauptmann, ``Infrared
	patch-image model for small target detection in a single image,'' \emph{IEEE
		Transactions on Image Processing}, vol.~22, no.~12, pp. 4996--5009, 2013.
	
	\bibitem{8727950}
	R.~Liu, S.~Cheng, Y.~He, X.~Fan, Z.~Lin, and Z.~Luo, ``On the convergence of
	learning-based iterative methods for nonconvex inverse problems,'' \emph{IEEE
		Transactions on Pattern Analysis and Machine Intelligence}, vol.~42, no.~12,
	pp. 3027--3039, 2020.
	
	\bibitem{huang2023adaptive}
	Z.~Huang, Z.~Zhang, C.~Lan, Z.-J. Zha, Y.~Lu, and B.~Guo, ``Adaptive frequency
	filters as efficient global token mixers,'' in \emph{ICCV}, 2023.
	
	\bibitem{MSHNet}
	Q.~Liu, R.~Liu, B.~Zheng, H.~Wang, and Y.~Fu, ``Infrared small target detection
	with scale and location sensitivity,'' in \emph{Proceedings of the IEEE/CVF
		Computer Vision and Pattern Recognition}, 2024.
	
	\bibitem{LCRNet}
	G.~Zhang, G.~Xu, S.~Chen, H.~Wang, and X.~Zhang, ``Learning dynamic local
	context representations for infrared small target detection,'' \emph{IEEE
		Transactions on Geoscience and Remote Sensing}, vol.~63, pp. 1--13, 2025.
	
	\bibitem{10061442}
	Z.~Zhang, K.~Chen, K.~Tang, and Y.~Duan, ``Fast multi-grid methods for
	minimizing curvature energies,'' \emph{IEEE Transactions on Image
		Processing}, vol.~32, pp. 1716--1731, 2023.
	
	\bibitem{9560049}
	Y.~Han, G.~Huang, S.~Song, L.~Yang, H.~Wang, and Y.~Wang, ``Dynamic neural
	networks: A survey,'' \emph{IEEE Transactions on Pattern Analysis and Machine
		Intelligence}, vol.~44, no.~11, pp. 7436--7456, 2022.
	
	\bibitem{NEURIPS2019_7716d0fc}
	T.~Liu, M.~Chen, M.~Zhou, S.~Du, E.~Zhou, and T.~Zhao, ``Towards understanding
	the importance of shortcut connections in residual networks,'' in
	\emph{Advances in Neural Information Processing Systems}, vol.~32, 2019.
	
	\bibitem{11069297}
	G.~Xu, G.~Zhang, L.~Ye, S.~Gan, X.~Zhang, and X.~Yang, ``Optimizing
	local-global dependencies for accurate 3d human pose estimation,'' \emph{IEEE
		Transactions on Circuits and Systems for Video Technology}, vol.~35, no.~12,
	pp. 12\,306--12\,316, 2025.
	
	\bibitem{ALCNet}
	Y.~Dai, Y.~Wu, F.~Zhou, and K.~Barnard, ``Attentional local contrast networks
	for infrared small target detection,'' \emph{IEEE Transactions on Geoscience
		and Remote Sensing}, vol.~59, no.~11, pp. 9813--9824, 2021.
	
	\bibitem{AGPCNet}
	T.~Zhang, L.~Li, S.~Cao, T.~Pu, and Z.~Peng, ``Attention-guided pyramid context
	networks for detecting infrared small target under complex background,''
	\emph{IEEE Transactions on Aerospace and Electronic Systems}, vol.~59, no.~4,
	pp. 4250--4261, 2023.
	
	\bibitem{chen2024frequency}
	L.~Chen, Y.~Fu, L.~Gu, C.~Yan, T.~Harada, and G.~Huang, ``Frequency-aware
	feature fusion for dense image prediction,'' \emph{IEEE Transactions on
		Pattern Analysis and Machine Intelligence}, vol.~46, no.~12, pp.
	10\,763--10\,780, 2024.
	
	\bibitem{TGRS2025DRPCANet}
	Z.~Xiong, F.~Zhou, F.~Wu, S.~Yuan, M.~Fu, Z.~Peng, J.~Yang, and Y.~Dai,
	``Drpca-net: Make robust pca great again for infrared small target
	detection,'' \emph{IEEE Transactions on Geoscience and Remote Sensing},
	vol.~63, pp. 1--16, 2025.
	
	\bibitem{yang2025pinwheel}
	J.~Yang, S.~Liu, J.~Wu, X.~Su, N.~Hai, and X.~Huang, ``Pinwheel-shaped
	convolution and scale-based dynamic loss for infrared small target
	detection,'' in \emph{Proceedings of the AAAI Conference on Artificial
		Intelligence}, vol.~39, no.~9, 2025, pp. 9202--9210.
	
	\bibitem{li2025u}
	C.~Li, X.~Liu, W.~Li, C.~Wang, H.~Liu, Y.~Liu, Z.~Chen, and Y.~Yuan, ``U-kan
	makes strong backbone for medical image segmentation and generation,'' in
	\emph{Proceedings of the AAAI conference on artificial intelligence},
	vol.~39, no.~5, 2025, pp. 4652--4660.
	
	\bibitem{11014512}
	Z.~Sun, H.~Wang, Q.~Xie, Y.~Zheng, and D.~Meng, ``Rsf-conv: Rotation-and-scale
	equivariant fourier parameterized convolution for retinal vessel
	segmentation,'' \emph{IEEE Transactions on Neural Networks and Learning
		Systems}, vol.~36, no.~9, pp. 16\,549--16\,563, 2025.
	
	\bibitem{zhou2025nnwnet}
	Y.~Zhou, L.~Li, L.~Lu, and M.~Xu, ``nnwnet: Rethinking the use of transformers
	in biomedical image segmentation and calling for a unified evaluation
	benchmark,'' in \emph{Proceedings of the Computer Vision and Pattern
		Recognition Conference}, 2025, pp. 20\,852--20\,862.
	
	\bibitem{10844040}
	C.~Hao, Z.~Yu, X.~Liu, J.~Xu, H.~Yue, and J.~Yang, ``A simple yet effective
	network based on vision transformer for camouflaged object and salient object
	detection,'' \emph{IEEE Transactions on Image Processing}, vol.~34, pp.
	608--622, 2025.
	
	\bibitem{chen2025frequency}
	L.~Chen, L.~Gu, L.~Li, C.~Yan, and Y.~Fu, ``Frequency dynamic convolution for
	dense image prediction,'' in \emph{Proceedings of the Computer Vision and
		Pattern Recognition Conference}, 2025, pp. 30\,178--30\,188.
	
	\bibitem{chen2025thernet}
	J.~Chen, S.~Shu, C.~Meng, and X.~Bai, ``Thernet: Thermal segmentation network
	harnessing physical properties,'' \emph{IEEE Transactions on Pattern Analysis
		and Machine Intelligence}, 2025.
	
	\bibitem{11297835}
	R.~Cong, Z.~Chen, H.~Fang, S.~Kwong, and W.~Zhang, ``Breaking barriers,
	localizing saliency: A large-scale benchmark and baseline for
	condition-constrained salient object detection,'' \emph{IEEE Transactions on
		Pattern Analysis and Machine Intelligence}, vol.~48, no.~4, pp. 4167--4183,
	2026.
	
	\bibitem{11515007}
	X.~Wang, F.~Yao, G.~Zhong, Q.~Cai, S.~Wang, and J.~T.-Y. Kwok, ``Gbnet: Gated
	boundary-aware network for camouflaged object detection,'' \emph{IEEE
		Transactions on Image Processing}, vol.~35, pp. 5297--5310, 2026.
	
	\bibitem{11373245}
	R.~Li, C.~Xiao, Q.~Yin, W.~An, N.~Chen, X.~Ying, M.~Li, and Y.~Wang, ``Dynamic
	high-frequency convolution for infrared small target detection,'' \emph{IEEE
		Transactions on Circuits and Systems for Video Technology}, vol.~36, no.~6,
	pp. 7676--7680, 2026.
	
	\bibitem{yuan2026seeing}
	M.~Yuan, D.~Meng, Z.~Xi, T.~Zhao, S.~Zhao, Y.~Dai, and X.~Wei, ``Seeing through
	the noise: Improving infrared small target detection and segmentation from
	noise suppression perspective,'' in \emph{Proceedings of the IEEE/CVF
		Conference on Computer Vision and Pattern Recognition}, 2026, pp.
	27\,783--27\,792.
	
	\bibitem{SeRankDet}
	Y.~Dai, P.~Pan, Y.~Qian, Y.~Li, X.~Li, J.~Yang, and H.~Wang, ``Pick of the
	bunch: Detecting infrared small targets beyond hit-miss trade-offs via
	selective rank-aware attention,'' \emph{IEEE Transactions on Geoscience and
		Remote Sensing}, vol.~62, pp. 1--15, 2024.
	
	\bibitem{visualloss}
	H.~Li, Z.~Xu, G.~Taylor, C.~Studer, and T.~Goldstein, ``Visualizing the loss
	landscape of neural nets,'' in \emph{Neural Information Processing Systems},
	2018.
	
	\bibitem{ZHANG2026113208}
	X.~Zhang, C.~Liu, T.~Song, D.~Yang, Y.~Ye, K.~Li, and Y.~Song, ``Rfaconv:
	Receptive-field attention convolution for improving convolutional neural
	networks,'' \emph{Pattern Recognition}, vol. 176, p. 113208, 2026.
	
	\bibitem{11230099}
	X.~Zheng, L.~Zhang, B.~Chen, M.~Wang, F.~Chao, C.~Wu, S.~Wang, R.~Ji, and
	Y.~Tian, ``Good performance estimation strategies are all you need in neural
	architecture search,'' \emph{IEEE Transactions on Pattern Analysis and
		Machine Intelligence}, vol.~48, no.~3, pp. 2398--2412, 2026.
	
	\bibitem{11568692}
	F.-L. Fan, M.~Wang, H.-C. Dong, J.~Ma, and T.~Zeng, ``No one-size-fits-all
	neurons: Task-based neurons for artificial neural networks,'' \emph{IEEE
		Transactions on Pattern Analysis and Machine Intelligence}, pp. 1--8, 2026.
	
	\bibitem{11599663}
	S.~Lee, ``Np-hardness of minimizing neurons in two-hidden-layer relu neural
	networks,'' \emph{IEEE Transactions on Pattern Analysis and Machine
		Intelligence}, pp. 1--9, 2026.
	
\end{thebibliography}

\newpage

\vfill

\end{document}